\documentclass[11pt]{article}
\usepackage[top=1in,bottom=1in,right=1in,left=1in]{geometry}
\usepackage{amssymb,amsmath,amsthm}
\usepackage{pgfplots}
\usepackage{pgfplotstable}
\usepackage{arydshln}
\usepackage{verbatim}
\usepackage{titlesec}
\usepackage{subfigure}
\usepackage[linesnumbered,ruled]{algorithm2e}
\usepackage{bm}
\usepackage{pdfpages}
\usepackage{enumerate} 
\usepackage{float}

\newcommand\blfootnote[1]{%
  \begingroup
  \renewcommand\thefootnote{}\footnote{#1}%
  \addtocounter{footnote}{-1}%
  \endgroup
}

\definecolor{forestgreen}{rgb}{0.13, 0.55, 0.13}
\usepackage[colorlinks, linkcolor=blue,citecolor=blue,urlcolor = black, bookmarks=true]{hyperref}
\usepackage[T1]{fontenc}
 
\usepackage[nameinlink, noabbrev, capitalize]{cleveref}
\usepackage{tikz}
\usepackage{libertine}

\pgfplotsset{compat=1.18}
\newcommand{\ra}[1]{\renewcommand{\arraystretch}{#1}}
\usepackage{booktabs} 
\usepackage{multirow} 
\usepackage{array}

\crefname{equation}{}{}
\crefrangeformat{equation}{(#3#1#4) --~(#5#2#6)}
\crefmultiformat{equation}{(#2#1#3)}{, (#2#1#3)}{, (#2#1#3)}{, (#2#1#3)}
\crefname{lem}{Lemma}{Lemmas}
\crefname{section}{Section}{Sections}
\crefname{subsubsubsection}{Section}{Sections}
\crefname{rem}{Remark}{Remarks}
\crefname{figure}{Figure}{Figures}
\crefname{table}{Table}{Tables}
\Crefname{lem}{Lemma}{Lemmas}
\crefname{thm}{Theorem}{Theorems}
\Crefname{thm}{Theorem}{Theorems}

\usepackage{commath}

\newtheorem{theorem}{Theorem}[section]
\newtheorem{lemma}[theorem]{Lemma}
\newtheorem{proposition}[theorem]{Proposition}
\newtheorem{corollary}[theorem]{Corollary}
\newtheorem{assumption}{Assumption}
\newtheorem{fact}[theorem]{Fact}
\newtheorem{claim}[theorem]{Claim}

\newtheorem{definition}[theorem]{Definition}

\theoremstyle{remark}
\newtheorem{remark}{Remark}

\title{Bypassing the Noisy Parity Barrier:\\ Learning Higher-Order Markov Random Fields from Dynamics\blfootnote{J.G. is supported by Vannevar Bush Faculty Fellowship ONR-N00014-20-1-2826 and Simons Investigator Award 622132. A.M. is supported in part by a Microsoft Trustworthy AI Grant, an ONR grant and a David
and Lucile Packard Fellowship. E.M. is supported in part by Vannevar Bush Faculty Fellowship ONR-N00014-20-1-2826, Simons Investigator Award 622132, and Simons-NSF DMS-2031883.}}
\author{Jason Gaitonde\\
Massachusetts Institute of Technology\\
\texttt{gaitonde@mit.edu}
\and Ankur Moitra\\
Massachusetts Institute of Technology\\
\texttt{moitra@mit.edu}
\and Elchanan Mossel\\
Massachusetts Institute of Technology\\
\texttt{elmos@mit.edu}}

\begin{document}

\maketitle
\thispagestyle{empty}

\begin{abstract}
We consider the problem of learning graphical models, also known as Markov random fields (MRFs) from {\em temporally} correlated samples.
As in many traditional statistical settings, fundamental results in the area all assume independent samples from the distribution. However, these samples generally will not directly correspond to more realistic observations from nature, which instead \emph{evolve} according to some stochastic process. From the computational lens, even generating a single sample from the true MRF distribution is intractable unless $\mathsf{NP}=\mathsf{RP}$, and moreover, any algorithm to learn from i.i.d. samples requires prohibitive runtime due to hardness reductions to the \emph{parity with noise problem}. These computational barriers for sampling and learning from the i.i.d. setting severely lessen the utility of these breakthrough results for this important task; however, dropping this assumption typically only introduces further algorithmic and statistical complexities.

 
In this work, we surprisingly demonstrate that the direct trajectory data from a natural evolution of the MRF overcomes the fundamental computational lower bounds to efficient learning. In particular, we show that given a trajectory with $\widetilde{O}_k(n)$ site updates of an order $k$ MRF from the \emph{Glauber dynamics}, a well-studied, natural stochastic process on graphical models, there is an algorithm that recovers the graph and the parameters in $\widetilde{O}_k(n^2)$ time. 
By contrast, all prior algorithms for learning order $k$ MRFs inherently suffer from $n^{\Theta(k)}$ runtime even in sparse instances due to the reductions to sparse parity with noise. Our results thus surprisingly show that this more realistic, but intuitively less tractable, model for MRFs actually leads to 
efficiency far beyond what is known and believed to be true in the traditional i.i.d. case. 
\end{abstract}

\newpage

\tableofcontents
\newpage
\setcounter{page}{1}
\section{Introduction}
Graphical models, or Markov Random Fields (MRFs), are a powerful way to represent high-dimensional distributions in terms of their conditional dependency structure. They are described by a dependency graph $G=(V,E)$ (also known as the \emph{Markov blanket}) where $\vert V\vert=n$ and to each vertex $v\in V$ we associate a random variable $X_v\in \{-1,1\}$.\footnote{In general, these random variables may belong to a larger, finite alphabet; our results can be extended with appropriate modification to this setting, so we focus on the binary case.}  The defining feature of an MRF $\mu$ is that the conditional independence structure of the variables is determined by the dependency graph: the random variable $X_v$ is conditionally independent of $X_{V\setminus (\mathcal{N}(v)\cup \{v\})}$ given $X_{\mathcal{N}(v)}$, where $\mathcal{N}(v)$ denote the graph-theoretic neighbors. 
Due to their widespread applications across statistics, mathematics, and engineering, the algorithmic problem of learning the underlying local structure or the actual parameters of an MRF has been studied for decades~\cite{chow_liu,DBLP:conf/soda/KargerS01,DBLP:journals/siamcomp/BreslerMS13,ravikumar2010high}, with a flurry of algorithmic advances in the last decade~\cite{DBLP:conf/stoc/Bresler15,DBLP:conf/nips/VuffrayMLC16, DBLP:conf/focs/KlivansM17, DBLP:conf/nips/HamiltonKM17, DBLP:conf/nips/WuSD19,DBLP:conf/icml/Zhang0KW20}.

As in most traditional statistical settings, these breakthrough works assume access to i.i.d. samples $X^1,\ldots,X^p\sim \mu$ to obtain algorithmic learning guarantees. However, the observations one can typically directly observe in nature are often generated by a \emph{stochastic evolution} of the underlying system, which exhibits strong temporal correlations and may not directly yield i.i.d. samples. But the i.i.d. assumption for MRFs also suffers more broadly from inherent computational barriers as a plausible generative model, not simply misalignment from natural observations. It is well-known that obtaining a single sample from $\mu$ when there are long-range correlations can be computationally hard unless $\mathsf{NP}=\mathsf{RP}$ \cite{sly2010computational, DBLP:conf/focs/SlyS12}---in this case, it is implausible that nature, or any other efficient process, could generate samples that approximate those from $\mu$. As a result, while we can often observe some correlated trajectory of a system described by the MRF, it is unrealistic to expect they resemble a typical instance from the distribution.

From the algorithmic perspective, learning from i.i.d. samples also encounters fundamental computational hardness lower bounds. In general, MRFs can be expressed in terms of \emph{clique potentials}, or functions defined on variables in a clique of $G$~\cite{clifford1990markov}, and the maximum size clique with a nonzero potential is the \emph{order}. While higher-order models are rich and expressive, state-of-the-art algorithms~\cite{DBLP:conf/focs/KlivansM17,DBLP:conf/nips/HamiltonKM17,DBLP:conf/icml/Zhang0KW20} for order $k$ MRFs run in time roughly $n^{\Theta(k)}$, precluding realistic implementation for even small values of $k$. This slowdown is widely-believed to be \emph{fundamental}, and not a lack of algorithmic ingenuity: an order $k+1$ MRF with just a single nonzero term can encode the notorious \emph{$k$-sparse parity with noise (SPN) problem}~\cite{DBLP:conf/nips/BreslerGS14a,DBLP:conf/focs/KlivansM17}, a natural barrier to computationally efficient learning across a variety of important settings \cite{DBLP:conf/stoc/MosselR05, DBLP:conf/stoc/BreslerKM19, block_mdp_spn, bruna2021continuous, gupte2022continuous}. While only $O(k\log(n/k))$ samples are needed information-theoretically, the best known algorithms for SPN require time $n^{c k}$ for some $1/2 \leq c < 1$ \cite{DBLP:conf/alt/GrigorescuRV11,DBLP:journals/jacm/Valiant15} with similar lower bounds in restricted models of computation \cite{blum1994weakly}. 

These multifaceted computational obstructions for learning MRFs in the traditional i.i.d. framework paint a rather pessimistic picture about their applicability, particularly in settings with long-range correlations, but abandoning the i.i.d. assumption typically only introduces more algorithmic and statistical complexities. In this work, we challenge this intuition by proving that: 

\begin{quote}
\emph{Observing the direct trajectory data
from the natural evolution of an MRF not only avoids the intractability of generating samples, but also overcomes these fundamental computational lower bounds to efficiently learn higher order MRFs.}
\end{quote}

In more detail, we consider the problem of learning MRFs from the \emph{Glauber dynamics}, a natural stochastic process in computer science, economics, and physics which has recently been considered by prior work on learning the simpler Ising model~\cite{DBLP:journals/tit/BreslerGS18,DBLP:conf/icml/DuttLVM21,unified}. Glauber dynamics are a popular Markov chain for sampling from high-dimensional distributions, but importantly are a natural observational model even when obtaining i.i.d. samples from $\mu$ is difficult since it asymptotically converges to $\mu$ under mild conditions, but may not do so efficiently. In the discrete-time version of this model, the initial configuration can be arbitrary. Then for each time $t=1,2,\ldots$, a uniformly random site $i_t\in [n]$ resamples their current value according to $\mu$ conditioned on the current values of the other sites (see \Cref{sec:glauber} for the continuous-time definition and \Cref{defn:discrete} for the formal discrete-time version). As in all prior work, we assume that both $i_t$ and the new configuration are observed at each time. 

Since these dynamics can be easily implemented in a decentralized manner using only local site interactions, they have been considered independently as natural, exogenous processes as ``noisy best-response dynamics'' \cite{kandori,BLUME1993387,young,DBLP:conf/focs/MontanariS09} in the economics literature or as ``heat bath dynamics'' in statistical physics for a particle system to converge towards equilibrium. But while these previous works develop new approaches that show that this model is not too much computationally and statistically harder than the standard i.i.d. setting, we extend far beyond these conclusions: these naturally correlated observations can be algorithmically \emph{harnessed} to break these longstanding computational barriers for learning MRFs.

Concretely, we first obtain an efficient structure recovery algorithm: given $O(n\log n)$ total site updates from a Glauber trajectory of $\mu$, one can correctly output the conditional dependency structure in time $O(n^2\log n)$ with high probability when each variable influences a bounded number of variables (i.e. has bounded degree). Once this structure has been identified, we show how to recover the parameters of the MRF to $\varepsilon$-additive accuracy with high probability using a trajectory of Glauber dynamics with $O(n\log(n) \cdot\mathsf{poly}(\log\log(n),1/\varepsilon))$ total site updates with additional $O(n\cdot \mathsf{poly}(\log n,1/\varepsilon))$ time. In both results, the hidden constants depend in standard ways on the order, variable degrees, and other non-degeneracy parameters of $\mu$, but these are completely decoupled from $n$. Therefore, the algorithmic guarantees of the dynamical setting completely avoid the fundamental $n^k$ type behavior in the i.i.d. setting. Conceptually, our results thus showcase the potential for observations with time dependencies to simultaneously be a more plausible model for data and yet obtain computational efficiency far beyond the limits of the traditional i.i.d. setting.

\subsection{Our Results}
\label{sec:results}
We now describe our problem formulation and results in more detail. We consider the problem of learning an MRF $\mu$ on $\{-1,1\}^n$ where $\mu(\bm{x})\propto \exp(\psi(\bm{x}))$ for a multilinear polynomial $\psi:\{-1,1\}^n\to \mathbb{R}$ (i.e. all variables have polynomial degree at most $1$). We define the associated dependency graph $G=(V,E)$ by $i\sim j$ or $(i,j)\in E$ if the mixed partial derivative satisfies $\partial_i\partial_j\psi\not\equiv 0$, so that $X_i$ and $X_j$ influence each other in some nonzero monomial of $\psi$. We also write $\mathcal{N}(i)$ for the set of neighbors of site $i$ in $G$. This formulation is equivalent to the more standard representation via clique potentials since any clique potential can be written uniquely as a multilinear function of the spins. We say that the polynomial degree of $\psi$ is the \emph{order} of the MRF $\mu=\mu_{\psi}$ and the special case that $\psi$ is quadratic ($k=2$) is the \emph{Ising model}.

We impose standard degree and non-degeneracy conditions on $\mu$ in \Cref{assumption:psi}. Namely, we assume that $\mu$ is a $(k,d,\alpha,\lambda)$-MRF, where $k$ is the order, $d$ is the graph-theoretic degree of the dependency graph, $\alpha$ lower bounds the effect of neighbors in the potential $\psi$, and $\lambda$ upper bounds the total influence of any variable in the potential. Our main result is the following guarantee for structure learning:

\begin{theorem}[\Cref{thm:structure_recovery}, informal]
\label{thm:intro_main}
    Let $\mu=\mu_{\psi}$ be a $(k,d,\alpha,\lambda)$-MRF. Then there exists an algorithm that, given $O_{k,d,\alpha,\lambda}(n\log n)$ total site updates of Glauber dynamics, outputs the dependency graph 
 $G$ of $\mu$ with high probability. The runtime is $O_{k,d,\alpha,\lambda}(n^2\log n)$.
\end{theorem}
\noindent The implicit constants are of the form $\mathsf{poly}(d^k,\exp(\lambda k),1/\alpha)$ (which qualitatively also appear in state-of-the-art algorithms for the i.i.d. setting~\cite{DBLP:conf/focs/KlivansM17,DBLP:conf/icml/Zhang0KW20}) with no further dependence on $n$. A prototypical setting of parameters is $k=O(1)$ a large constant, $\alpha=\Omega(1)$ a small constant, while $d,\lambda$ are bounded or very slowly growing in $n$. This regime is particularly well-suited to network and statistical physics applications, where the number of neighbors of a site is bounded or only slowly growing in the system size. We further show in \Cref{sec:app_lb} that our algorithm cannot be substantially simplified in the sense that any alternative approach must use either compute more difficult ``non-local'' statistics or use more ``local'' samples to perform learning in an idealized sampling model.

In fact, while \Cref{thm:intro_main} provides a stark \emph{computational} improvement over what is possible from i.i.d. samples, our analysis shows that learning from dynamics also yields an \emph{information-theoretic benefit} when learning from models with unobserved variables. We show that while it is impossible to determine whether there are any dependencies among the other variables in $K$ from i.i.d. samples if one cannot observe all variables, our analysis of learning from Glauber dynamics will nonetheless recover the induced subgraph of $G$ on the observed variables. We elaborate on this point in \Cref{sec:unobserved}.

Once the dependency graph $G$ is recovered, we then show how to efficiently recover the actual parameters of the MRF. We prove the following result:
\begin{theorem}[\Cref{thm:lr_final}, informal]
    \label{thm:main_logreg}
    Let $\mu=\mu_{\psi}$ be a $(k,d,\alpha,\lambda)$-MRF with known dependency graph $G$. Then given $\varepsilon>0$ and $O_{k,d,\lambda}(\log(n)\cdot \mathsf{poly}(\log\log n,1/\varepsilon))$ site updates of Glauber dynamics for each node, applying node-wise logistic regression to a suitable sequence of update times for each node $i\in [n]$ yields a polynomial $\widehat{\psi}$ such that $\|\widehat{\psi}-\psi\|_{\infty}\leq \varepsilon$ (coefficient-wise closeness) in time $n\cdot O_{k,d,\lambda}(\mathsf{poly}(\log n,1/\varepsilon))$ with high probability.
\end{theorem}
\noindent In \Cref{thm:main_logreg}, the implicit constants are of the form $d^{O(\lambda k^2)}$ (if we assume $\lambda=\Omega(1)$), which remains bounded for the important setting where $d,\lambda,k=O(1)$ large constants. Again, the dependence on these parameters is completely decoupled from dimension. 

Taken together, \Cref{thm:intro_main} and \Cref{thm:main_logreg} collectively show that Glauber trajectories, while often a more realistic model of real-world observations, naturally balance two competing objectives to enable end-to-end efficiency of learning MRFs. Their temporal correlations can isolate the challenging structural dependencies, but the dynamics remain sufficiently random to efficiently identify the precise parameters once these dependencies are known. Both of these features in tandem are essential in overcoming the computational barriers of learning from the i.i.d. setting.

\subsection{Other Related Work}

\noindent\textbf{Learning MRFs.} The problem of learning undirected graphical models, particularly the Ising model, has been the subject of intense study in the statistics and computer science literature. Early work on provable guarantees for learning from i.i.d. samples from the Ising model~\cite{DBLP:journals/siamcomp/BreslerMS13,ravikumar2010high} required very high-temperature properties like incoherence or correlation decay for efficient learnability. The breakthrough work of Bresler ~\cite{DBLP:conf/stoc/Bresler15} provided efficient algorithms for bounded-degree graphs even for low-temperature models; as mentioned above, these guarantees have since been improved in degree dependence~\cite{DBLP:conf/nips/VuffrayMLC16,DBLP:conf/focs/KlivansM17,DBLP:conf/nips/WuSD19,unified} and have been extended to MRFs~\cite{DBLP:conf/focs/KlivansM17,DBLP:conf/nips/HamiltonKM17,DBLP:conf/icml/Zhang0KW20}. Of these, the work of Klivans and Meka~\cite{DBLP:conf/focs/KlivansM17} was the first to provide learning guarantees in terms of the more general $\ell_1$ width conditions; these guarantees are essentially tight due to information-theoretic lower bounds of Santhanam and Wainwright~\cite{DBLP:journals/tit/SanthanamW12}. The work of Devroye, Mehrabian, and Reddad~\cite{devroye2020minimax} provides tight minimax rates for learning Ising models in distribution.

Several variations and specializations of the i.i.d. framework have been studied. For instance, in the case that the underlying Ising model is a tree, the classical work of Chow and Liu~\cite{chow_liu} provides an efficient algorithm for learning the maximum likelihood tree. More refined guarantees for this setting have recently been obtained by numerous papers; see for instance, the work of ~\cite{bresler_karzand,DBLP:journals/siamcomp/BhattacharyyaGPTV23,DBLP:conf/stoc/DaskalakisP21,DBLP:conf/colt/KandirosDDC23,DBLP:conf/focs/Boix-AdseraBK21} among others. In the case that the model has latent variables, it is well-known that the problem of learning from i.i.d. samples from $\mu$ is quite challenging. The early work of  Bogdanov, Mossel, and Vadhan~\cite{DBLP:conf/approx/BogdanovMV08} has shown computational hardness in this setting, though there are improved guarantees under restrictions on these models~\cite{anandkumar2013learning,DBLP:conf/stoc/BreslerKM19,DBLP:conf/aistats/Goel20,DBLP:conf/nips/GoelKK20}. In our model, learning the dependencies among the observable subset of variables comes `for free' in our dynamical setting.

By contrast, the investigation of alternative sampling models in light of the computational barriers to generating i.i.d. samples from $\mu$ has been comparatively less well studied, but the primary takeaway has been that these correlated models are computationally \emph{similar} (if not slightly worse) to the i.i.d. case. The early work of Bresler, Gamarnik, and Shah~\cite{DBLP:journals/tit/BreslerGS18} provides a simple structure learning algorithm for the Ising model from continuous time Glauber dynamics, as we consider here. They complement these guarantees with similar information-theoretic lower bounds as the i.i.d. setting, thus establishing dynamics and i.i.d. samples are comparable for structure learning when $k=2$. Recent work of Dutt, Lokhov, Vuffray, and Misra~\cite{DBLP:conf/icml/DuttLVM21} shows empirically that existing algorithms, though without formal guarantees, indeed have similar complexity to that of the i.i.d. setting, corroborating these theoretical findings of Bresler, Gamarnik, and Shah~\cite{DBLP:journals/tit/BreslerGS18}. The recent work of Gaitonde and Mossel~\cite{unified} extends these existing insights for parameter learning, showing that the complexity is at most worse by $\mathsf{poly}(\log\log n)$ factors due to the use of more involved martingale tail bounds.\\

\noindent\textbf{Learning from Dynamics.} Our work fits into a broader theme of \emph{learning from dynamics}; a comprehensive survey of lines that fall under this framework is beyond the scope of this paper, so we highlight just a few. A classical application of this paradigm is the problem of learning \emph{linear dynamical systems}, a foundational problem in control theory. Learning the driving matrices of a linear dynamical system from input-output pairs is known as \emph{system identification} and has a storied history since the seminal work of Kalman \cite{Kalman}. We defer to the recent work of Bakshi, Liu, Moitra, and Yau \cite{DBLP:conf/stoc/BakshiLMY23} for more discussion on algorithmic results for this problem. Recent work has considered the problem of \emph{quantum Hamiltonian structure learning from trajectories}; we defer to the very recent work of Bakshi, Liu, Moitra, and Tang \cite{bakshi2024structure} for more discussion on this problem. A crucial difference between their observation model and ours is that their result requires the ability to apply quantum gates and observe the resulting process; in our model, the dynamics are natural, exogenous, and uncontrollable. A related learning problem is that of learning properties of network structure from (i.i.d.) cascade trajectories; see, for instance ~\cite{DBLP:conf/sigmetrics/NetrapalliS12,DBLP:conf/kdd/AbrahaoCKP13,DBLP:journals/pomacs/HoffmannC19,DBLP:conf/colt/MosselS24}.

A related model of learning from dynamical observations that is similar in spirit to our model is \emph{learning Boolean functions from random walks}, as introduced in the work of Bartlett, Fischer, and H{\"{o}}ffgen \cite{DBLP:journals/iandc/BartlettFH02}. Rather than learning a Boolean function $f:\{-1,1\}^n\to \mathbb{R}$ from i.i.d. samples $(\bm{x}_t,f(\bm{x}_t))$ where $\bm{x}_t\sim \{-1,1\}^n$, it is instead assumed that the process $\bm{x}_1,\bm{x}_2,\ldots$ follows Glauber dynamics on the hypercube (i.e. a single uniformly random coordinate re-randomizes at each time). Their work provides efficient learning algorithms under this model for simple, specific classes. This early work was extended by Bshouty, Mossel, O'Donnell, and Servedio~\cite{DBLP:journals/jcss/BshoutyMOS05} and Jackson and Wimmer~\cite{DBLP:journals/jmlr/JacksonW14}. Our problem is incomparable since our goal is to learn the high-dimensional dependencies encoded in \emph{noisy transitions} of a complex random walk, rather than learn noiseless functions of a simple, known random walk. To see the difference, note that if $f$ is a parity, determining which coordinates $i\in [n]$ are relevant is trivial from random walk samples since the sign of $f$ will flip exactly when a relevant coordinate flips in the random walk on $\{-1,1\}^n$. In contrast, learning an instance of SPN in our setting is already nontrivial since the effect of flipping a relevant coordinate on the bias of the other relevant coordinates is noisy and moreover, will not immediately manifest due to the randomness in the dynamics. Thus, even identifying relevant variables requires different techniques.\\

\section{Technical Overview}
\label{sec:overview}
We now provide a high-level overview of our algorithm and analysis for structure learning under \Cref{assumption:psi}, as well as our lower bounds. We then describe how once the lower-dimensional structure is learned, it is possible to efficiently learn the actual parameters using node-wise logistic regression by combining several  techniques from prior work to account for the temporal dependencies while exploiting the randomness of the dynamics. We provide all formal definitions in \Cref{sec:prelims}.\\

\noindent\textbf{MRF Structure Learning: Heuristics and Challenges.}
To perform structure learning, one must be able to efficiently solve the following task: for given $i,j\in [n]$ is $i\sim j$ in the dependency graph $G$? Our algorithm and analysis for \Cref{thm:intro_main} relies on the following intuition (which is also utilized by Bresler, Gamarnik, and Shah~\cite{DBLP:conf/nips/BreslerGS14a}, henceforth BGS, for the simpler Ising model): to test whether $i\sim j$, one should look at \emph{localized} updates where $i$ and $j$ update close to each other to detect a statistical difference in the conditional law of $X_i$. Indeed, if $i\sim j$, one expects that the distribution of $X_i$ differs at nearby updates when $X_j=1$ and $X_j=-1$, since the rest of the configuration should be approximately fixed at these nearby updates. The key challenge is to find the right statistic which reveals the dependencies for higher order MRFs over the course of the dynamics.

In the Ising model considered by BGS, there is a simple approach that works. For any fixed pair $i,j\in [n]$, there is an \emph{unambiguous local relationship between $X_i$ and $X_j$}. Since the Ising model only has pairwise interactions by definition, if $i\sim j$, $X_i$ either always has a conditional preference to match signs with $X_j$ or always has a conditional preference to disagree in signs with $X_j$, once one conditions on the remaining spins. Conversely, if $i\not\sim j$, this conditional law never depends on the value of $X_j$. While the precise conditional distribution will depend on the value of the rest of $i$'s graph-theoretic neighbors during the trajectory, this insight naturally leads to their overall approach: condition on $X_j$ flipping signs between two nearby updates of $X_i$ and test whether $X_i$ has a \emph{consistent} preference to match in signs or a \emph{consistent} preference to disagree in signs with $X_j$.

Once one considers higher order MRFs beyond the Ising case with $k>2$, we face an immediate challenge: the local influence of $j$ on $i$ is no longer an unambiguous constant. Indeed, the local influence of $j$ on $i$ at some time $t$ is  $\partial_i\partial_j \psi(X^t)$ which is now some polynomial of degree $k-2>0$ evaluated on evolving spin configurations. At one extreme, this polynomial can fluctuate significantly over the course of the dynamics due to these latent variables, so that there is no consistent local influence---as we describe momentarily, this issue of variables confounding the local relationship leads to provable impossibility results for the type of statistic considered by BGS. At the other extreme, this polynomial may instead remain very close to zero over most of the dynamics even if the coefficients are nonzero, so that there is often \emph{no} local influence one can statistically test. Therefore, the key algorithmic challenge becomes:
\begin{quote}
    \emph{What efficiently computable local statistic, observable from dynamics, can identify the presence of local influences in higher order MRFs even when they often vanish or fluctuate sign?}
\end{quote}

\noindent\textbf{MRF Structure Learning: Algorithm Intuition.} Our main algorithmic contribution is in showing that all of these difficulties can be overcome by computing a surprisingly simple statistic that only requires \emph{two local updates of $X_i$ before and after an update of $X_j$, no matter the order of the MRF}. 

To motivate our construction, consider a conceptually related problem called \emph{testing populations of means}, which relates to very early work of Stein \cite{Stein+1956+197+206} and proceeds as follows. Suppose there is a distribution $\mathcal{D}$ over $\theta\in [0,1]$, and one receives the following kinds of samples. At each time $\ell=1,\ldots,L$, $\theta_{\ell}\sim \mathcal{D}$ is sampled independently, then one receives $X_1^{\ell},\ldots,X_t^{\ell}\sim \mathsf{Bern}(\theta_{\ell})$. For a given pair $(t,L)$, what can one statistically infer about $\mathcal{D}$? The interesting setting for this problem is small $t\in \mathbb{N}$ with $L$ not too large. In this regime, it is statistically impossible to estimate each $\theta_{\ell}$ to decent accuracy since this requires $t$ to be quite large.

It is easy to see that if $t=1$, it is information-theoretically impossible to distinguish between $\mathcal{D}=\delta_{1/2}$, the point mass at $1/2$ and $\mathcal{D}=\mathsf{Unif}([0,1])$. In general, the classical \emph{method of moments} implies one can estimate the first $t$ moments of $\mathcal{D}$ using $(t,L)$ observations for large $L$, which was exploited by a recent work of Tian, Kong, and Valiant~\cite{DBLP:conf/nips/TianKV17} to learn $\mathcal{D}$ in Wasserstein distance. But we can already directly observe even $t=2$ yields nontrivial information about the structure of $\mathcal{D}$; obtaining second moments of $\mathcal{D}$ is necessary and sufficient to determine whether or not $\mathcal{D}$ is a point mass, a drastic improvement on the failure of $t=1$ sample for each draw. 

Connecting this problem to  our setting, we can view the randomness over the values of $X^t_{\mathcal{N}(i)\setminus \{j\}}$ at certain stopping times of Glauber dynamics where we have nearby $X_i$ and $X_j$ updates as inducing some random conditional bias on $X_i$ (the analogue of samples from $\mathcal{D}$) at the nearby update times of site $i$, which is now a \emph{function} of $X_j$. Whether or not $i\sim j$ is equivalent to determining if this conditional bias \emph{ever has any dependence on the value of $X_j$}. For structure learning of MRFs, we can similarly show in \Cref{thm:local_lb} that under an idealized observation model that only removes the technical complications of Glauber trajectories, it is also \emph{information-theoretically impossible} to determine whether $i\sim j$ under \Cref{assumption:psi} with $t=1$ local sample. We prove this impossibility result by demonstrating that even in MRFs with $k=3$, the existence of a confounding variable $X_k$ can obfuscate the local influence of $i$ on $j$ on average over the randomness in the distribution. More precisely, we construct two distinct MRFs, one with $i\sim j$ and one with $i\not\sim j$, such that the distribution of $X_i$ given $X_j=\pm1$, averaged over the remaining randomness in the other variables sampled from $\mu$, is indistinguishable in both models. We remark that the BGS statistic for the Ising model only requires $t=1$ local sample, so our result rules out this approach.

But motivated by this connection, we show that 
for all $k$ simultaneously, just $t=2$ conditional samples of $X_i$ before and after a nearby update of $X_j$ in a short window suffice to perform structure learning. As in the population of means problem, one cannot hope to estimate the conditional bias of $X_i$ for each value of $X_j=\pm1$ for an approximately fixed outside configuration, since this requires a prohibitive number of $i$ and $j$ updates in a short interval. But as with detecting point masses, $t=2$ samples turns out to be both necessary and sufficient for local statistics to efficiently determine whether the random conditional bias of $X_i$ \emph{ever} depends on $X_j$ from Glauber trajectories. 

\noindent\textbf{MRF Structure Learning: Algorithmic Results and Analysis.} We now describe our exact construction and how it overcomes the issues described above. More formally, to test whether or not $i\sim j$, we define a sequence of stopping times $\tau_1,\tau_2,\ldots$ that roughly correspond to occurrences where we observe an update subsequence of the form $iijii$ (with no intermediate $j$ updates) in a short time interval as before. 



Our main theoretical result is that once these stopping times are appropriately constructed, the value of the following simple statistic will distinguish the cases that $i\sim j$ and $i\not\sim j$ across all $j\in [n]$. At each such stopping time $\tau$ where this update subsequence $iijii$ occurs, let $Y_1$ and $Y_2$ denote the indicator that $X^{t}_i=1$ at the first two update times, and $Y_1'$ and $Y_2'$ denote the indicator that $X^{t}_i=1$ at the last two update times, respectively. We then define the following random variable:

\begin{equation}
\label{eq:z_intro}
    Z = Y_1Y_2-2Y_1Y'_1+Y'_1Y_2'.
\end{equation}

The key heuristic is as follows: if no other neighbors of $i$ or $j$ update in the short window where $iijii$ was observed, $Y_1$ and $Y_2$ \emph{should} be independent Bernoulli random variables with some conditional probability $p_1$ while $Y_1',Y_2'$ \emph{should} be independent Bernoulli random variables with some other conditional probability $p_2$ depending on if $X_j$ changes signs. Both statements will hold under suitable, but \emph{different} conditionings. We then hope that, at least with some constant probability, we can argue using the following sequence of inequalities: 

\begin{align}
    \Omega(1)&\stackrel{?}{<}(p_1-p_2)^2 \label{eq:eff_large}\\
    &=p_1^2-2p_1p_2+p_2^2\nonumber\\
    &=\mathbb{E}[Y_1]\mathbb{E}[Y_2]-2\mathbb{E}[Y_1]\mathbb{E}[Y_1']+\mathbb{E}[Y_1']\mathbb{E}[Y_2']\nonumber\\
    &\stackrel{?}{\approx}\mathbb{E}[Y_1Y_2-2Y_1Y_1'+Y_1'Y_2'] \label{eq:heuristic}\\
    &=\mathbb{E}[Z].\nonumber
\end{align}

If so, we can simply output $i\sim j$ for all pairs whose empirical statistics are noticeably positive when aggregaged across stopping times. But at (\ref{eq:eff_large}), we immediately encounter one of the aforementioned challenges that the local influence of $j$ on $i$ could be negligible. It is easily seen that
\begin{equation}
\label{eq:anti_to}
    \vert p_1-p_2\vert>\Omega(1)\iff\vert \partial_i\partial_j \psi(X^{\tau})\vert>\Omega(1).
\end{equation} 
In particular, we must  contend with the possibility that $\partial_i\partial_j\psi\left(X^t_{\mathcal{N}(i)}\right)\approx 0$ often during the dynamics, even if some coefficients of this polynomial are assumed to be bounded away from zero. If this occurs, then it will be information-theoretically impossible to statistically test whether $X_j$ exerts any influence on the conditional law of $X_i$ at nearby update times.
Thus, (\ref{eq:eff_large}) requires ensuring that, with some constant probability that does not depend on $n$, it holds that this mixed partial derivative that gives the local influence of $j$ on $i$ is non-negligible.

To do so, we appeal to \emph{anti-concentration} properties of low-degree polynomials under sufficiently random variables. In general, the joint law of sites along some part of the Glauber trajectory may have a somewhat complex evolution, as each site update depends on the previous information and in turn influences future site updates. Therefore, we identify a sufficient notion of unpredictability of the variable values that provably holds during the dynamics: they can be shown, in a certain sense, to be \emph{Santha-Vazirani sources} (see \Cref{defn:sv}), an influential notion of randomness from the pseudorandomness literature~\cite{DBLP:conf/focs/SanthaV84} that is well-adapted to the sampling procedure of Glauber dynamics. We then can extend existing anti-concentration bounds for low-degree polynomials for stronger notions of randomness~\cite{DBLP:conf/focs/KlivansM17} to argue that \Cref{eq:anti_to} holds a noticeable fraction of the time, depending again on model parameters but not $n$. Our precise anti-concentration bounds for low-degree polynomials that take as input these less structured random variables are elementary, but end up yielding comparable guarantees to what is obtained under stronger notions of variability.

To justify \Cref{eq:heuristic}, an essential point in the construction is the use of $Y_1$ (rather than $Y_2$) in the cross-term in \Cref{eq:z_intro} to break subtle correlations even after conditioning on intermediate site updates. We remark that the precise approach of BGS in the Ising case already suffers from a similar hidden correlation that muddies the analysis. More precisely, conditioning on site $j$ flipping sign between nearby updates of $X_i$ is subtler than suggested above since the value of an update of $i$ can influence whether site $j$ flipped signs. Therefore, conditioning on $j$ flipping signs at an intermediate point can will affect the natural analysis of the conditional law of the updates of $X_i$ before it. To our knowledge, the precise BGS algorithm suffers from this correlation issue, though it is possible to alter the algorithm and analysis to overcome this hurdle.

In our case, we show that our construction itself circumvents this problem. If no neighbor of $i$ nor $j$ updates in any intermediate updates on our event, we argue that the cross term variables become \emph{conditionally independent} by the Markov property. Unlike the BGS approach, this occurs because intuitively, the intermediate site update of $j$ in the $iijii$ pattern cannot ``see'' the value of $Y_1$, only those of the later $X_i$ updates, since this value gets re-randomized by $Y_2$ before affecting any adjacent site. Therefore, a careful appeal to the Markov property and rewriting the argument without explicitly conditioning on $X_j$'s updated value until needed will justify \Cref{eq:heuristic}. Defining the stopping times properly will account for the possible sources of errors in these calculations, as well as ensure the requisite anti-concentration, and thus complete the analysis. 
A simpler argument for $i\not\sim j$ will ensure $\mathbb{E}[Z]\approx 0$ since $p_1=p_2$ so long as the other neighbors of $i$ do not change on this event, establishing the separation needed to distinguish adjacency.

To justify the runtime analysis, we appeal to the following reasoning also used by BGS. We show that one can pick the window on which the $iijii$ pattern occurs to lie in a short window of $cn$ site updates for a suitably small constant $c>0$ depending on the degree assumptions. It then only takes $\Theta(n)$ site updates to observe an event of this type for a fixed pair $i\neq j$, but still no neighbor of $i$ or $j$ will also update during this window with high probability if $c>0$ is small enough as a function of $d$. Unlike the Ising case, the window size $c>0$ should also be chosen sufficiently small in terms of other model parameters, but not $n$, to ensure that the (small) gap induced on the relatively low-probability (but constant) event when \Cref{eq:anti_to} holds significantly exceeds the error incurred on the event a neighbor updates. Since this error can be driven to zero by taking $c$ sufficiently small (but constant in $n$), the analysis will be mostly unaffected and only $O(n\log n)$ site updates overall to observe enough events for all $i,j$ pairs simultaneously by standard concentration bounds. The algorithm that simply aggregates \Cref{eq:z_intro} for each $i,j$ pair thus needs $O(n\log n)$ updates and $O(n^2\log n)$ time to correctly output $G$. Quantitatively, the implicit constants in this argument based on model parameters essentially resemble that of the i.i.d. setting~\cite{DBLP:conf/focs/KlivansM17}, but completely decouples these parameters from $n$.

Note that the definition of the statistic in \Cref{eq:z_intro} does not depend on the order $k$. The dependence on $d,k,\lambda$ will only appear in the analysis of the quantitative separation of \Cref{eq:anti_to} and thus the precise construction of the stopping times when we account for quantitative sources of error. In light of our lower bounds, our results demonstrate that while the Ising model with $k=2$ is fundamentally different from $k>2$, the ``complexity'' of the learning problem from dynamics remains well-behaved for $k>2$. Because the statistic only relies on the ``local'' effect of observed sites, the analysis also remains valid in determining dependencies between the observed variables even when some sites are latent, which is impossible in the i.i.d. setting.

\noindent\textbf{Parameter Recovery via Logistic Regression.} Once the Markov blanket $G$ has been recovered under \Cref{assumption:psi} via \Cref{thm:intro_main}, we then show how to recover the actual parameters from dynamical samples efficiently using node-wise logistic regression in \Cref{thm:lr_final}. The analysis of logistic regression for the simpler Ising model was recently proven by Gaitonde and Mossel~\cite{unified}. We show how to extend these techniques in tandem with analytic bounds from prior work on learning MRFs in the i.i.d. setting. The computational saving arises from the fact that (i) the logistic regression problem associated with any node $i\in [n]$ under \Cref{assumption:psi} is a convex program only over $\Theta(d^{k-1})$ polynomial coefficients once $\mathcal{N}(i)$ is known, and (ii) Glauber observations remain sufficiently unpredictable to enable parameter recovery, as in the i.i.d. setting.

The natural approach to logistic regression for a node $i\in [n]$ is simply to extract samples at suitable stopping times where node $i$ updated in the Glauber trajectory, as was done by Gaitonde and Mossel~\cite{unified} for the Ising model. To establish that node-wise logistic regression succeeds for parameter recovery on a suitable subset of the observations, we combine several techniques from the prior literature:

\begin{enumerate}[(i)]
    \item First, we show that the empirical logistic losses at these stopping times for all possible coefficient vectors for site $i$'s interactions uniformly converge to  \emph{random, population} logistic losses given enough samples with high probability. This analysis essentially follows that of Gaitonde and Mossel \cite{unified} and proceeds via the powerful equivalence between the sequential Rademacher complexity of Rakhlin, Sridharan, and Tewari and uniform martingale tail bounds~\cite{DBLP:journals/jmlr/RakhlinST15,DBLP:conf/colt/RakhlinS17}.\label{item:gm}
    \item Then, by deterministic inequalities of Wu, Sanghavi, and Dimakis \cite{DBLP:conf/nips/WuSD19}, it can be shown that the difference in these random population losses of the empirical optimizer $\widehat{p}$ and the true interactions $p^*:=\partial_i \psi$ will be lower bounded by the difference in their \emph{predictions} of the conditional bias of $X_i$ at these stopping times given the values of the sites in $\mathcal{N}(i)$, in expectation. \label{item:wsd}
    \item  Finally, by technical results established by Klivans and Meka \cite{DBLP:conf/focs/KlivansM17} for learning MRFs in the i.i.d. case, if the optimizer $\widehat{p}$ of the logistic regression problem and $p^*$ make similar predictions of the conditional bias of $X_i$ in expectation over \emph{any} unbiased distribution (see \Cref{defn:unbiased}), then $\|p^*-\widehat{p}\|_1$ must be small. \label{item:km}
\end{enumerate}

To leverage these facts, we must define suitable stopping times where site $i$ updates; this requires more care than was needed for the Ising model, where one can simply use all update times as samples for similar reasons as the relative simplicity of BGS for structure learning. We define them in a way that ensures that the law of the neighbors $\mathcal{N}(i)$ at these stopping times is quite likely to be $\delta$-unbiased (see \Cref{defn:unbiased}) for a suitable value of $\delta$ depending only on $d,k,\lambda$. We do so by generalizing an argument of Gaitonde and Mossel~\cite{unified} that bounds the posterior likelihood ratio of spin values for a given site conditioned on the rest of the trajectory. Whether or not unbiasedness holds crucially depends only on the sequence of updating \emph{indices}, but not the actual values of the site updates. Therefore, working on this event only affects the randomness in the choice of site updates, but not the conditional law of each (fixed) site update in the dynamics. Our analysis obtains $\delta=c\exp(-C\lambda k\ln d)=cd^{-C\lambda k}$ compared to $\delta = c\exp(-C\lambda)$ in the i.i.d. setting~\cite{DBLP:conf/focs/KlivansM17} due to the need to account for influences of site updates along the dynamics, but importantly this remains constant in $n$. The uniform convergence bounds from~(\ref{item:gm}) will imply that the empirical logistic loss optimizer $p^*$ will have low excess population logistic losses with high probability, and therefore~(\ref{item:wsd}) and~(\ref{item:km}) will establish parameter closeness to $\partial_i\psi$ due to the likely $\delta$-unbiasedness of neighbors at the stopping times.\\

\noindent\textbf{Empirical Results.} We complement our theoretical results with experiments in \Cref{sec:experiments} demonstrating that the problem of learning from dynamics is indeed computationally far easier than that of learning from i.i.d samples. In particular, we show that the algorithm of \Cref{thm:intro_main} succeeds in reasonable time to recover the dependence graph for SPN instances on moderate size instances. We suspect that with more careful tuning of parameters, the algorithm should succeed on much larger instances in practice with graceful runtime overhead. We compare the runtime for dynamical structure learning to that of the Sparsitron algorithm of Klivans and Meka~\cite{DBLP:conf/focs/KlivansM17} given i.i.d. samples. Even though this is arguably the most lightweight algorithm for i.i.d. learning, we find that their algorithm is unable to even approximately recover the dependence structure for these moderate size SPN instances in reasonable time, as to be expected under the conjectural hardness of this problem.

\section{Preliminaries}
\label{sec:prelims}

\paragraph{Notation}
We consider MRFs on $\{-1,1\}^n$. We use capital letters $X,Y,\ldots$ to denote random variables and bold font $\bm{x},\bm{y},\ldots$ to denote non-random vectors. We will use the notation $\mathcal{A},\mathcal{B},\ldots$ to denote events. We write $\mathcal{E}^c$ to denote the complement of the event $\mathcal{E}$. Given a subset of indices $S\subseteq [n]$, we use the subscript $-S$ to denote the restriction of a vector to the coordinates outside $S$.

\subsection{Polynomials}
For a subset $S\subseteq [n]$, we write $\bm{x}^S=\prod_{i\in S}x_i$. Given a function $f:\{-1,1\}^n\to \mathbb{R}$ with unique multilinear expansion $f(\bm{x})=\sum_{S\subseteq [n]} \widehat{f}(S)\bm{x}^S$, we define the Fourier $p$-norms $\|f\|_1:=\sum_{S\subseteq [n]} \vert \widehat{f}(S)\vert$ and $\|f\|_{\infty}=\max_{S\subseteq [n]}\vert \widehat{f}(S)\vert$.  We write $\mathsf{deg}(f)$ to denote the maximum degree of $f$ as a multilinear polynomial and $\mathsf{supp}(f)$ to denote the set of variables that appear in the multilinear expansion. We say $S$ is a maximal monomial of $f$ if $\widehat{f}(S)\neq 0$ while if $S\subsetneq T$, then $\widehat{f}(T)=0$. We write $\partial_i f$ to denote the $i$th partial derivative of $f$, which does not depend on $x_i$ by multilinearity. We say $f$ is a $d$-junta if $\vert\mathsf{supp}(f)\vert\leq d$. We write $i\sim j$ if $\partial_i f$ depends on $x_j$; note this relation is indeed symmetric since this occurs if and only if $\partial_i\partial_j f\equiv \partial_j\partial_i f\not\equiv 0$. For a subset $S\subseteq [n]$, we write $\mathcal{N}(S):=\{i\in [n]:i\sim j\text{ for some $j\in S$}\}$ for the set of sites that influence some site in $S$.
\subsection{Markov Random Fields}
Let $\psi:\{-1,1\}^n\to \mathbb{R}$ be a Hamiltonian, which we may uniquely write as a multilinear function $\psi(\bm{x})=\sum_{S\subseteq [n]} \widehat{\psi}(S)\bm{x}^S$ where $\bm{x}^S=\prod_{i\in S} x_i$. We consider the MRF $\mu_{\psi}$ with weights given by $
\mu(\bm{x})\propto \exp(\psi(\bm{x})).$ We make the following standard assumptions about the degree and non-degeneracy of the Markov random fields we consider.
\begin{assumption}[$(k,d,\alpha,\lambda)$-Markov Random Fields]
\label{assumption:psi}
    A Markov random field with Hamiltonian $\psi$ is a $(k,d,\alpha,\lambda)$-MRF if the following holds:
    \begin{enumerate}
        \item (Low Polynomial Degree) $\mathsf{deg}(\psi)\leq k$.
        \item (Bounded Vertex Degree) For every $i\in [n]$, $\partial_i \psi$ is a $d$-junta. Equivalently, $\vert \mathcal{N}(i)\vert=\vert\mathsf{supp}(\partial_i \psi)\vert\leq d$.
        \item (Nontrivial Edge Coefficients) If $i\sim j$, there exists a maximal monomial $S$ of $\psi$ such that $i,j\in S$ and satisfies $\vert \widehat{\psi}(S)\vert\geq \alpha$.

        \item (Bounded Width) For every $i\in [n]$, $\|\partial_i\psi\|_{1}=\sum_{S\ni i} \vert \widehat{\psi}(S)\vert\leq \lambda$.
    \end{enumerate}
\end{assumption}

\subsection{Continuous-Time Glauber Dynamics}
\label{sec:glauber}
For convenience, we first consider the continuous-time Glauber dynamics for a MRF $\mu$, which is a random process $(X^t)_{t\in \mathbb{R}_{\geq 0}}\in (\{\pm 1\})^{\mathbb{R}_{\geq 0}}$ defined as follows. $X^0\in \{-1,1\}^n$ is an arbitrary (possibly random) initial configuration. We assume each site $i\in [n]$ updates at each time $t$ in an independent Poisson process $\Pi_i\subseteq \mathbb{R}_{\geq 0}$ of rate $1$. In particular, on any interval $I\subseteq \mathbb{R}_{\geq 0}$,
\begin{equation}
\label{eq:prob_no_update}
    \Pr(\Pi_i\cap I= \emptyset)=\exp(-\vert I\vert),
\end{equation}
where $\vert I\vert$ is the length of $I$. Note that $\Pi_i\cap \Pi_j=\emptyset$ almost surely for all $i\neq j$ so there is no ambiguity over which node updates, if any, at a given time $t$. For a measurable subset $I\subseteq \mathbb{R}$, we write $\Pi_i(I)=\Pi_i\cap I$ for the sequence of update times of node $i$ in $I$. Note that $\Pi_i(I_1)$ and $\Pi_j(I_2)$ are independent if either $I_1$ and $I_2$ are disjoint (up to measure zero) or if $i\neq j$. For convenience, we write $\Pi_i(t_1,t_2)$ as shorthand for $\Pi_i([t_1,t_2])$ and $\Pi_i(t)$ as shorthand for $\Pi_i([0,t])$.

Let $\mathcal{F}_t=\sigma((X^{t'})_{t'=0}^t,\{\Pi_i(t)\}_{i\in [n]})$ be the $\sigma$-algebra (history) generated by the update times and actual updates of $X$ until time $t\geq 0$, where we assume that the sets $\Pi_i$ are marked by the label $i\in [n]$ so that there is no ambiguity over which node updates. We let $\sigma(z):=\frac{1}{1+\exp(-z)}$ denote the sigmoid function. Given any $i\in [n]$ and configuration $\bm{x}_{-i}\in \{-1,1\}^{n-1}$, the Glauber update at site $i$ given that $X_{-i}^t=\bm{x}_{-i}$ and $t\in \Pi_i$ has the conditional law:
\begin{equation}
\label{eq:glauber}
    \Pr(X^t_i=1\vert X^t_{-i}=\bm{x}_{-i},t\in \Pi_i)=\sigma(2\partial_i\psi(\bm{x}_{-i})),
\end{equation}
where we recall that $\partial_i\psi(\bm{x})$ does not depend on $x_i$. As with prior work, we assume that we observe the random variables generating $\mathcal{F}_t$ at time $t$, including site updates whether or not the value changes.

We require the following simple estimates on the probabilities that a subset of variables is or is not updated in a given interval:

\begin{lemma}
\label{lem:update_bounds}
Let $S\subseteq [n]$ be a subset of size $\ell$. Fix an interval $I\subseteq \mathbb{R}_{\geq 0}$ of length $T$ and let $U_I$ denote the set of sites that are ever chosen for updating in $I$. Then it holds that:
\begin{gather*}
    \Pr(S\subseteq U_I)=(1-\exp(-T))^{\ell}\geq 1-\ell\exp(-T),\\
    \Pr(S\cap U_I=\emptyset)= \exp(-T\ell).
\end{gather*}
\end{lemma}
\begin{proof}
    Both statements follow directly from \Cref{eq:prob_no_update} and either using independence or a union bound.
\end{proof}

The following simple lower bound on the conditional probability that a node updates to $\pm 1$ given any outside configuration is classical:
\begin{fact}
\label{fact:prob_lb}
Under \Cref{assumption:psi}, given that $i\in [n]$ is chosen for updating at some time $t\geq 0$, it holds for each $\varepsilon\in \{-1,1\}$ and any $\bm{z}\in \{-1,1\}^{n-1}$ that
\begin{equation*}
    \Pr(X^t_i=\varepsilon\vert X^t_{-i}=\bm{z})\geq \frac{\exp(-2\lambda)}{2}.
\end{equation*}
\end{fact}

We also require the following lower bounds on the strict monotonicity of $\sigma$.

\begin{fact}[\cite{DBLP:conf/focs/KlivansM17}]
\label{fact:km_sigmoid}
    For any $x,y\in \mathbb{R}$, $
        \vert \sigma(x)-\sigma(y)\vert\geq \exp(-\vert x\vert-3)\min\{1,\vert x-y\vert\}.
$  Moreover, suppose that $\vert x\vert,\vert y\vert\leq \lambda$. Then $\vert \sigma(x)-\sigma(y)\vert\geq \frac{\exp(-\lambda)}{4}\cdot \vert x-y\vert
$.
\end{fact}

\subsection{Unpredictable Distributions}
For our results, we will require the following two notions of conditional variability of sites in an arbitrary distribution.

\begin{definition}[Unbiased Distributions]
\label{defn:unbiased}
    A distribution $\mu$ on $\{-1,1\}^n$ is $\delta$-unbiased if for any $i\in [n]$ and any $\bm{x}\in \{-1,1\}^{n-1}$, it holds that $\delta\leq \Pr_{X\sim \mu}(X_i=1\vert X_{-i}=\bm{x})\leq 1-\delta.$
\end{definition}

In words, an unbiased distribution is such that each coordinate maintains some lower bounded variance for \emph{every} conditioning of the other variables.

\begin{definition}[Santha-Vazirani Source~\cite{DBLP:conf/focs/SanthaV84}]
\label{defn:sv}
    A distribution $\mu$ on $\{-1,1\}^k$ is a $\delta$-Santha-Vazirani source with respect to a permutation $\sigma:[k]\to [k]$ if for each $t\leq k$, it surely holds that $
        \delta\leq \Pr_{\mu}(X_{\sigma(t)}=\varepsilon\vert X_{\sigma(1)},\ldots,X_{\sigma(t-1)})\leq 1- \delta.$
\end{definition}

A Santha-Vazirani source (with respect to some ordering) is such that one can sample variables in some order such that for each time, the sampling has lower bounded variance. It is easy to see that an unbiased distribution is Santha-Vazirani with respect to any order, but simple examples show that the converse is false.\footnote{Consider a random variable whose first coordinate is uniform in $\{-1,1\}$, and then each remaining coordinate agrees with the first with $1-\delta$ probability independently. The Chernoff bound certifies that $x_1=\mathrm{sgn}(\sum_{i=2}^nx_i)$ with all but exponentially small probability for any fixed $\delta>0$.}

We need the following two anti-concentration results for polynomials. The first is well-known from the work of Klivans and Meka \cite{DBLP:conf/focs/KlivansM17} and holds for any $\delta$-unbiased distribution. The latter holds for distributions induced by Glauber dynamics on MRFs.

\begin{lemma}[Lemma 6.1 of~\cite{DBLP:conf/focs/KlivansM17}]
    Suppose that $\mu$ is a $\delta$-unbiased distribution on $\{-1,1\}^n$ and let $f:\{-1,1\}^n\to \mathbb{R}$ be a multilinear polynomial. Suppose $I\subseteq [n]$ is a maximal nonzero monomial in $f$. Then
    \begin{equation*}
        \Pr_{X\sim \mu}\left(\vert f(X)\vert\geq \vert \widehat{f}(I)\vert\right)\geq \delta^{\vert I\vert}.
    \end{equation*}
\end{lemma}

\noindent We will also crucially use the following anti-concentration result for observations generated by dynamics. We defer the proof to \Cref{sec:app1} as well as discussion on its tightness.
\begin{lemma}
\label{lem:anti-concentration_main}
    Let $f:\{-1,1\}^n\to \mathbb{R}$ be supported on $[d]$ with degree at most $k$ and let $S$ be a maximal monomial of $f$ and $T> 0$. Suppose that $\mu=\mu_{\psi}$ is an MRF such that the conditional distribution of any site with any outside configuration is uniformly lower bounded by $\delta$. Let $\mu^T$ be the law of $X^T$ after running continuous-time Glauber dynamics on $I=[0,T]$ with some arbitrary initial configuration $X^0$. Further, let $\mathcal{E}_S$ denote the event that every $i\in S$ is updated by the dynamics (i.e. $\Pi_i\cap I\neq \emptyset$ for all $i\in S$). Then for any $T'\geq 0$,
    \begin{equation*}
        \Pr_{X\sim \mu^{T}}\left(\vert f(X^T)\vert\geq \vert \widehat{f}(S)\vert\bigg\vert \mathcal{E}_S,\{\Pi_j(T')\}_{j\in [n]\setminus [d]}\right)\geq \left(\frac{\delta}{d}\right)^k.
    \end{equation*}
\end{lemma}

\section{Efficient Structure Learning of MRFs from Dynamics}
In this section, we provide our main result, \Cref{alg:markov_blanket}, and prove the correctness and the runtime guarantees of \Cref{thm:intro_main}. In \Cref{sec:stopping_times}, we formally define the stopping times and establish a number of elementary probabilistic bounds on the occurrence of suitable events at these stopping times. In \Cref{sec:adj_stat}, we introduce the simple statistic that distinguishes between neighbors and non-neighbors in $G$; we do so by leveraging the probabilistic inequalities established in \Cref{sec:stopping_times}. We then present \Cref{alg:markov_blanket} and prove correctness in \Cref{sec:alg_final} using the quantitative bounds of the previous subsection. In \Cref{sec:unobserved}, we explain why the guarantees extend to the setting where some subset of variables are unobserved. Finally, in \Cref{sec:app_lbs}, we show that \Cref{alg:markov_blanket} is essentially the simplest possible approach to recovering $G$ from dynamical samples in a slightly idealized observation model.

\subsection{Stopping Times and Filtrations}
\label{sec:stopping_times}
To state our algorithm and guarantees, we require the following extra notation. Throughout this section, we will fix a pair $i\neq j\in [n]$. We will also fix parameters $0<L<1/3$ and $r\in \mathbb{N}$ to be defined later. We will work under \Cref{assumption:psi} for the remainder of this section even if not explicitly stated.

For $\ell=0,1,\ldots,$ define the following filtration:
\begin{equation*}
\mathcal{G}_{\ell}=\sigma\left(\mathcal{F}_{\ell\cdot L},\Pi_i((\ell+r)\cdot L),\Pi_j((\ell+r)\cdot L)\right).
\end{equation*}
In words, we partition continuous time into consecutive intervals of length $L$ and consider the revealment filtration generated by the full Glauber trajectory until $\ell\cdot L$, \emph{as well as the update times (but not values) of just sites $i$ and $j$ for an extra $r$ blocks}. We then define $I^{(\ell)}=[\ell\cdot L,(\ell +r)\cdot L]$ for the contiguous block of $r$ intervals of length $L$ starting with the $\ell$'th interval. We further partition this interval as
\begin{equation*}
    I^{(\ell)}=I^{(\ell)}_1\sqcup I^{(\ell)}_2:=\underbrace{[\ell\cdot L,(\ell+r-1)\cdot L]}_{I^{(\ell)}_1}\sqcup \underbrace{[(\ell+r-1)\cdot L,(\ell+r)\cdot L]}_{I^{(\ell)}_{2}}.
\end{equation*}

Finally, we define $I_2^{(\ell)}=I_{2,1}^{(\ell)}\sqcup I_{2,2}^{(\ell)}\sqcup I_{2,3}^{(\ell)}$ in the natural way where $I_{2,1}^{(\ell)}$ is the first $L/3$ part of $I_2^{(\ell)}$, $I_{2,2}^{(\ell)}$ is the middle $L/3$ part, and $I_{2,3}^{(\ell)}$ is the last $L/3$ part of $I_2^{(\ell)}$. With this notation, we can equivalently write $\mathcal{G}_{\ell}=\sigma\left(\mathcal{F}_{\ell\cdot L},\Pi_i(I^{(\ell)}),\Pi_j(I^{(\ell)})\right)$. 

We now define the following events for each $\ell\geq 0$:
\begin{gather*}
    \mathcal{A}^{(\ell)}=\left\{\left\vert \partial_i\partial_j\psi\left(X^{(\ell+r-1)\cdot L}\right)\right\vert\geq \alpha\right\}\\
    \mathcal{B}^{(\ell)}=\left\{\bigcup_{k\in \mathcal{N}(\{i,j\})\setminus\{i,j\}} \Pi_k(I_2^{(\ell)})=\emptyset\right\}\\
    \mathcal{C}^{(\ell)}=\left\{\vert\Pi_i(I_{2,1}^{(\ell)})\vert\geq 2, \Pi_j(I_{2,1}^{(\ell)})=\emptyset\right\}\cap \left\{ \Pi_i(I_{2,2}^{(\ell)})=\emptyset,\vert\Pi_j(I_{2,2}^{(\ell)})\vert\geq 1,\right\}\cap \left\{\vert\Pi_i(I_{2,3}^{(\ell)})\vert\geq 2, \Pi_j(I_{2,3}^{(\ell)})=\emptyset\right\}.
\end{gather*}
In words, these events have the following interpretation:
\begin{itemize}
    \item $\mathcal{A}^{(\ell)}$ is the event that the $(i,j)$-mixed partial derivative of $\psi$ is large (in absolute value) at the end of $I^{(\ell)}_1$ (equivalently, beginning of $I^{(\ell)}_2$). This event corresponds to node $j$ having a large effect on the conditional law of site $i$, given the other neighbors. The interval $I^{(\ell)}_{1}$ ensures enough unpredictability (\Cref{defn:sv}) of the coordinates in $\mathcal{N}(i)\setminus \{j\}$ for the probability of $\mathcal{A}^{(\ell)}$ given $\mathcal{G}_{\ell}$ to be lower bounded.
    \item $\mathcal{B}^{(\ell)}$ is the event that no neighbor of either site $i$ or $j$ updates in $I_2^{(\ell)}$. This event, when it holds, will ensure that the influence of Glauber updates of $i$ and $j$ in this interval remains controlled.
    \item $\mathcal{C}^{(\ell)}$ is the event that on $I_2^{(\ell)}$, node $i$ updates at least \emph{twice} on the first third while node $j$ does not update, node $i$ does not update on the middle third while node $j$ does, and finally node $i$ updates at least \emph{twice} on the last third while node $j$ does not update.
\end{itemize}

Note that of these events, only $\mathcal{C}^{(\ell)}$ is measurable with respect to the filtration $\mathcal{G}_{\ell}$ by construction.

Finally, we recursively define the stopping times for each $s\geq 1$:
\begin{equation}
\label{eq:stopping_times}
    \tau_1 = \min\left\{\ell\geq r: \mathcal{C}^{(\ell)} \text{ occurs}\right\},\quad
    \tau_{s+1}=\min\left\{\ell\geq \tau_s+r: \mathcal{C}^{(\ell)} \text{ occurs}\right\}.
\end{equation}
It is easy to check that these are valid stopping times since the event $\mathcal{C}^{(\ell)}$ is measurable with respect to $\mathcal{G}_{\ell}$ by construction. The role of $r$ is to ensure there are enough updates that the event $\mathcal{A}^{(\tau_{\ell})}$ has lower bounded conditional probability given $\mathcal{G}_{\tau_{\ell-1}}$ if $i\sim j$. We will later see that the gaps between these stopping times are stochastically dominated by suitable geometric random variables up to the additive $r$.

We now prove various facts about this process and these events. The first is a conditional lower bound on the probability that $\mathcal{A}^{(\tau_{\ell})}$ occurs at each stopping time given the revealment filtration.

\begin{lemma}
\label{lem:burn_in}
    Suppose that \Cref{assumption:psi} holds and suppose further that $i\sim j$. Suppose that $(r-1)\cdot L\geq \log(\max\{1,2(k-2)\})$ Then for any $\ell\geq 1$, it holds that
    \begin{equation*}
        \Pr\left(\mathcal{A}^{(\tau_{\ell})}\bigg\vert \mathcal{G}_{\tau_{\ell}}\right)\geq \frac{1}{2}\left(\frac{\exp(-2\lambda)}{2d}\right)^{k-2}:=q_{\ref{lem:burn_in}}.
    \end{equation*}
\end{lemma}
\begin{proof}
    Let $f=\partial_i\partial_j \psi$. By \Cref{assumption:psi}, it holds that $f\not\equiv 0$ and moreover, is a multilinear polynomial of degree at most $k-2$ that depends on at most $d$ variables. Moreover, \Cref{assumption:psi} implies there is a maximal monomial $S$ (of size at most $k-2$) such that $\vert \widehat{f}(S)\vert\geq \alpha$.

    Let $\mathcal{E}^{(\tau_{\ell})}_S$ denote the event that each site in $S$ updates at least once during $I^{(\tau_{\ell})}_1$. By \Cref{lem:update_bounds} and the independence of this event from $\mathcal{G}_{\tau_{\ell}}$ (since it depends only on the independent update times outside $\{i,j\}$), it holds that
    \begin{equation*}
        \Pr(\mathcal{E}^{(\tau_{\ell})}_S\vert \mathcal{G}_{\tau_{\ell}})\geq 1/2
    \end{equation*}
    by our choice of the length of $I^{(\tau_{\ell})}_1$. \Cref{fact:prob_lb} and \Cref{lem:anti-concentration_main} thus imply that
    \begin{equation*}
        \Pr\left(\mathcal{A}^{(\tau_{\ell})}\big\vert \mathcal{G}_{\tau_{\ell}}\right)\geq \Pr\left(\mathcal{A}^{(\tau_{\ell})}\big\vert \mathcal{E}^{(\tau_{\ell})}_S,\mathcal{G}_{\tau_{\ell}}\right)\cdot \Pr(\mathcal{E}^{(\tau_{\ell})}_S\vert \mathcal{G}_{\tau_{\ell}})\geq \frac{1}{2}\left(\frac{\exp(-2\lambda)}{2d}\right)^{k-2}.
    \end{equation*}
    Here, we use the fact that by the Markov property, the dynamics conditional on $\mathcal{G}_{\tau_{\ell}}$ is equal in law to the dynamics with initial configuration $X^{\tau_{\ell}\cdot L}$ given the update times (but importantly, not the values) of sites $i$ and $j$ in the interval. 
\end{proof}

In light of \Cref{lem:burn_in}, given $L>0$ (which we will set later), we define:
\begin{equation}
    r:=\left\lceil \frac{\log(2\max\{1,2(k-2)\})}{L}\right\rceil
\end{equation}
We now prove a simple lower bound on the conditional probability of $\mathcal{B}^{(\tau_{\ell})}$ at the stopping time that is immediate from the definition of continuous-time Glauber dynamics:

\begin{lemma}
\label{lem:no_nbr_lb}
    For any $\ell\geq 1$,
    \begin{equation*}
        \Pr\left(\mathcal{B}^{(\tau_{\ell})}\vert \mathcal{G}_{\tau_{\ell}}\right)\geq \exp(-2dL).
    \end{equation*}
\end{lemma}
\begin{proof}
    As mentioned above, since $\mathcal{B}^{(\tau_{\ell})}$ only depends on the update times on $I^{(\tau_{\ell})}_2$ for $\mathcal{N}(\{i,j\})\setminus \{i,j\}$, this event is independent of $\mathcal{G}_{\tau_{\ell}}$ since this conditions only on the update times in this interval for $i$ and $j$. Since $\vert \mathcal{N}(\{i,j\})\setminus \{i,j\}\vert\leq 2d$ by \Cref{assumption:psi}, the lower bound is an immediate consequence of \Cref{lem:update_bounds}.
\end{proof}

Finally, we prove a lower bound on the probability of the occurrence of $\mathcal{C}^{(\ell)}$ on each window, conditioned on any trajectory up to the beginning of $I^{(\ell)}_2$:

\begin{lemma}
\label{lem:good_event_lb}
        For any $\ell\geq 0$, it holds that
    \begin{equation*}
        \Pr(\mathcal{C}^{(\ell)}\vert \mathcal{F}_{(\ell+r-1)\cdot L})\geq \frac{L^5}{6^5}:= q_{\ref{lem:good_event_lb}}.
    \end{equation*}
\end{lemma}
\begin{proof}
    For any interval $I'$ of length $L/3$, \Cref{lem:update_bounds} implies
    \begin{gather}
        \Pr\left(\left\vert\Pi_i\cap I'\right\vert\geq 2\land \Pi_j\cap I'=\emptyset \right)\geq\left(1-\exp(-L/6)\right)^2\cdot \exp(-L/3)\\
        \Pr\left(\left\vert\Pi_j\cap I'\right\vert\geq 1\land \Pi_i\cap I'=\emptyset \right)=\left(1-\exp(-L/3)\right)\cdot \exp(-L/3)
    \end{gather}
    The first inequality holds since the event that $\vert \Pi_i\cap I'\vert\geq 2$ is contained in the event that $\Pi_i$ intersects each disjoint half of $I'$, and each of these events are independent on intervals of length $L/6$. By the inequality $1-\exp(-x)\geq \exp(-x)x$ for any $x\geq 0$, and since $\Pi_k\cap I^{(\ell)}_2$ for each $k$ are independent of each other and $\mathcal{F}_{\ell\cdot L}$ by disjointness, multiplying these out give
    \begin{equation*}
        \Pr(\mathcal{C}^{(\ell)}\vert \mathcal{F}_{(\ell+r-1)\cdot L})\geq \exp(-2L)\cdot\frac{L^5}{3\cdot 6^4}\geq \frac{L^5}{6^5},
    \end{equation*}
    where we use $L\leq 1/3$ in the last inequality to lower bound $\exp(-2L)\geq 1/2$.
\end{proof}

\subsection{Adjacency Statistics}
\label{sec:adj_stat}
Given $\ell\geq 0$, we define the following statistic $Z^{(\ell)}$ when $\mathcal{C}^{(\ell)}$ occurs. Let $Y^{(\ell)}_1$ and $Y^{(\ell)}_2$ denote the indicator that $X^{t}_i=1$ at the \emph{first and second} update times of site $i$ on $I^{(\ell)}_{2,1}$. By the definition of $\mathcal{C}^{(\ell)}$, these update times exist and are distinct. Similarly, let $Y^{(\ell)'}_1$ and $Y^{(\ell)'}_2$ denote the indicator that $X^t_i=1$ at the \emph{first and second} update times of site $i$ on $I^{(\ell)}_{2,3}$. We define the following statistic:

\begin{equation}
\label{eq:z_def}
    Z^{(\ell)} = Y^{(\ell)}_1Y^{(\ell)}_2-2Y^{(\ell)}_1Y^{(\ell)'}_1+Y^{(\ell)'}_1Y_2^{(\ell)'}.
\end{equation}
Notice that $Z^{(\ell)}\in[-2,2]$ surely when $\mathcal{C}^{(\ell)}$ occurs. 

We will consider the behavior of the sequence $Z^{(\tau_{\ell})}$ for $\ell=1,\ldots$. Clearly $\mathcal{C}^{(\tau_{\ell})}$ occurs by definition of the stopping time, so $Z^{(\tau_{\ell})}$ is well-defined. Note further that for all $\ell\geq 1$, the random variable $Z^{(\tau_{\ell})}$ is measurable with respect to $\mathcal{G}_{\tau_{\ell+1}}$ since $Z^{(\tau_{\ell})}$ is measurable with respect to $\mathcal{F}_{(\tau_{\ell}+r)\cdot L}\subseteq \mathcal{G}_{\tau_{\ell+1}}$ by the recursive construction of the stopping times.

We now establish a quantitative separation in the expected value of this statistic when $i\sim j$ and $i\not\sim j$. This analysis is the heart of the algorithm for structure learning. We first treat the case that $i\sim j$:

\begin{proposition}[Neighbor Lower Bound]
\label{prop:nbr_lb}
    Suppose that $i\sim j$ under \Cref{assumption:psi}. Then
    \begin{equation}
    \label{eq:overall_lb}
        \mathbb{E}[Z^{(\tau_{\ell})}\vert \mathcal{B}^{(\tau_{\ell})},\mathcal{G}_{\tau_{\ell}}]\geq \frac{q_{\ref{lem:burn_in}}\alpha^2}{8}\exp(-6\lambda).
    \end{equation}
    More specifically, it holds that
    \begin{gather}
    \label{eq:first_lb}
        \mathbb{E}\left[Z^{(\tau_{\ell})}\vert  \mathcal{A}^{(\tau_{\ell})}\cap \mathcal{B}^{(\tau_{\ell})},\mathcal{G}_{\tau_{\ell}}\right]\geq \frac{\alpha^2}{8}\exp(-6\lambda)\\
        \label{eq:weak_lb}
        \mathbb{E}\left[Z^{(\tau_{\ell})}\vert (\mathcal{A}^{(\tau_{\ell})})^c\cap \mathcal{B}^{(\tau_{\ell})},\mathcal{G}_{\tau_{\ell}}\right]\geq 0.
    \end{gather}
\end{proposition}
\begin{proof}
    For convenience, let $X^0_{\mathcal{N}(i)}$ denote the initial configuration of spins at the start of the interval $I^{(\tau_{\ell})}_2$ on the neighbors of sites $i$. We will write $X^{0,\mathsf{flip}}_{\mathcal{N}(i)}$ to denote the configuration where the value of site $j$ is flipped from $X^0$. We will carry out the analysis further conditioned on $X^0_{\mathcal{N}(i)\setminus \{j\}}$. Note that the occurrence of $\mathcal{A}^{\tau_{\ell}}_{ij}$ depends only on $X^0_{\mathcal{N}(i)\setminus \{j\}}$ by definition, so we may derive uniform lower bounds for the expectation for when this event is satisfied or not depending on whether $X^0_{\mathcal{N}(i)\setminus \{j\}}$ satisfies $\mathcal{A}^{\tau_{\ell}}_{ij}$ at the end.
    
    By the definition of $\mathcal{B}^{(\tau_{\ell})}$, the values of the sites in $\mathcal{N}(i)$ other than $j$ never change throughout $I_2^{(\tau_{\ell})}$ since no such site is updated. Thus, since $j$ also does not update by the definition of $\mathcal{C}^{(\tau_{\ell})}$, conditioned just on $X^0_{\mathcal{N}(i)}$ and $\mathcal{B}^{(\tau_{\ell})}$, $Y_1^{(\tau_{\ell})}$ and $Y_2^{(\tau_{\ell})}$ are independent Bernoulli random variables with probability $\sigma(2\partial_i\psi(X^0_{\mathcal{N}(i)}))$ of being $1$. Thus
    \begin{equation*}
        \mathbb{E}[Y^{(\tau_{\ell})}_1Y^{(\tau_{\ell})}_2\vert X^0_{\mathcal{N}(i)}, \mathcal{B}^{(\tau_{\ell})},\mathcal{G}_{\tau_{\ell}}]=\sigma(2\partial_i\psi(X^0_{\mathcal{N}(i)}))^2.
    \end{equation*}

    For the other two terms, we argue as follows. We have
    \begin{equation*}
        \mathbb{E}[Y^{(\tau_{\ell})}_1Y_1^{(\tau_{\ell})'}\vert X^0_{\mathcal{N}(i)}, \mathcal{B}^{\tau_{\ell}}_{ij},\mathcal{G}_{\tau_{\ell}}]=\mathbb{E}\left[\mathbb{E}\left[Y_1^{(\tau_{\ell})}Y_1^{(\tau_{\ell})'}\vert Y_2^{(\tau_{\ell})},X^0_{\mathcal{N}(i)}, \mathcal{B}^{(\tau_{\ell})},\mathcal{G}_{\tau_{\ell}}\right]\bigg\vert X^0_{\mathcal{N}(i)}, \mathcal{B}^{(\tau_{\ell})},\mathcal{G}_{\tau_{\ell}}\right].
    \end{equation*}
    
    The key observation is that since no other neighbor of site $j$ updates in $I_2^{(\tau_{\ell})}$ on $\mathcal{B}^{(\tau_{\ell})}$ by definition, $Y_1^{(\tau_{\ell})}$ and $Y_1^{(\tau_{\ell})'}$ are \emph{conditionally independent} given $X^0$, $\mathcal{B}^{(\tau_{\ell})}$, $\mathcal{G}_{\tau_{\ell}}$, as well as $Y_2^{(\tau_{\ell})}$. Indeed, given $X^0$, $\mathcal{B}^{(\tau_{\ell})}$, $\mathcal{G}_{\tau_{\ell}}$, $Y^{(\tau_{\ell})}_{1}$ is independent from the other site updates of $i$ on $I_{2,1}^{(\ell)}$ since there are no site updates in $\mathcal{N}(i)$ on this interval that can be affected by this value by definition of $\mathcal{B}^{(\tau_{\ell})}$ and $\mathcal{C}^{(\tau_{\ell})}$. Moreover, the conditional law of all updates of site $j$ on $I_{2,2}^{(\ell)}$ depend only on the last $i$ update on $I_{2,1}^{(\ell)}$ (which is distinct and thus independent of $Y_1^{(\tau_{\ell})}$ by the existence of $Y_2^{(\tau_{\ell})}$), and the conditional law of all updates of site $i$ on $I_{2,3}^{(\ell)}$ depend only on the last update of site $j$ on $I_{2,2}^{(\ell)}$. Therefore, $Y^{(\tau_{\ell})}_{1}$ is conditionally independent from all other updates of site $i$ and $j$ given $X^0$, $\mathcal{B}^{(\tau_{\ell})}$, $\mathcal{G}_{\tau_{\ell}}$ even when further conditioned on $Y_2^{(\tau_{\ell})}$.
    
    Therefore, let $p=p(Y_2^{(\tau_{\ell})})$ denote the conditional probability that site $j$ did not flip values from the initial configuration $X^0$ at the end of $I_{2,2}^{(\tau_{\ell})}$ given $X^0$, $\mathcal{B}^{(\tau_{\ell})}$, $\mathcal{G}_{\tau_{\ell}}$, as well as $Y_2^{(\tau_{\ell})}$. We thus find from this discussion that:
    \begin{align*}
        \mathbb{E}\left[\mathbb{E}\left[Y_1^{(\tau_{\ell})}Y_1^{(\tau_{\ell})'}\vert Y_2^{(\tau_{\ell})},X^0_{\mathcal{N}(i)}, \mathcal{B}^{(\tau_{\ell})},\mathcal{G}_{\tau_{\ell}}\right]\bigg\vert X^0_{\mathcal{N}(i)}, \mathcal{B}^{(\tau_{\ell})},\mathcal{G}_{\tau_{\ell}}\right]&=\mathbb{E}[\mathbb{E}\left[Y_1^{(\tau_{\ell})}\vert Y_2^{(\tau_{\ell})},X^0_{\mathcal{N}(i)}, \mathcal{B}^{(\tau_{\ell})},\mathcal{G}_{\tau_{\ell}}\right]\\
        &\cdot\mathbb{E}\left[Y_1^{(\tau_{\ell})'}\vert Y_2^{(\tau_{\ell})},X^0_{\mathcal{N}(i)}, \mathcal{B}^{(\tau_{\ell})},\mathcal{G}_{\tau_{\ell}}\right]\big\vert X^0_{\mathcal{N}(i)}, \mathcal{B}^{(\tau_{\ell})},\mathcal{G}_{\tau_{\ell}}]\\
        &=\mathbb{E}[(1-p(Y_2^{(\tau_{\ell})}))\sigma(2\partial_i\psi(X^{0}_{\mathcal{N}(i)}))\sigma(2\partial_i\psi(X^{0,\mathsf{flip}}_{\mathcal{N}(i)}))\\
        &+p(Y_2^{(\tau_{\ell})})\sigma(2\partial_i\psi(X^{0}_{\mathcal{N}(i)}))^2\vert X^0_{\mathcal{N}(i)}, \mathcal{B}^{(\tau_{\ell})},\mathcal{G}_{\tau_{\ell}}]
    \end{align*}

    For the last term of \Cref{eq:z_def}, an analogous (and simpler) argument using conditional independence of the samples in $I_{2,3}^{(\tau_{\ell})}$ given $Y_2^{(\tau_{\ell})}$ on these events yields
    \begin{equation*}
        \mathbb{E}[Y^{(\tau_{\ell})'}_1Y_2^{(\tau_{\ell})'}\vert  X^0_{\mathcal{N}(i)}, \mathcal{B}^{(\tau_{\ell})},\mathcal{G}_{\tau_{\ell}}]=\mathbb{E}\left[p(Y_2^{(\tau_{\ell})})\sigma(2\partial_i\psi(X^{0}_{\mathcal{N}(i)}))^2+(1-p(Y_2^{(\tau_{\ell})}))\sigma(2\partial_i\psi(X^{0,\mathsf{flip}}_{\mathcal{N}(i)}))^2\vert X^0_{\mathcal{N}(i)}, \mathcal{B}^{(\tau_{\ell})},\mathcal{G}_{\tau_{\ell}}\right].
    \end{equation*}

    Putting the above together and writing the (random) convex combination
    \begin{equation*}
        \sigma(2\partial_i\psi(X^0_{\mathcal{N}(i)}))^2=p(Y_2^{(\tau_{\ell})})\sigma(2\partial_i\psi(X^0_{\mathcal{N}(i)}))^2+(1-p(Y_2^{(\tau_{\ell})}))\sigma(2\partial_i\psi(X^0_{\mathcal{N}(i)}))^2,
    \end{equation*}
    we obtain
    \begin{align*}
        \mathbb{E}\left[Z^{(\tau_{\ell})}\vert X^0_{\mathcal{N}(i)},\mathcal{B}^{(\tau_{\ell})},\mathcal{G}_{\tau_{\ell}}\right]&=\mathbb{E}\bigg[p\sigma(2\partial_i\psi(X^0_{\mathcal{N}(i)}))^2+(1-p)\sigma(2\partial_i\psi(X^0_{\mathcal{N}(i)}))^2\\
        &-2\left(p\sigma(2\partial_i\psi(X^{0}_{\mathcal{N}(i)}))^2+(1-p)\sigma(2\partial_i\psi(X^{0}_{\mathcal{N}(i)}))\sigma(2\partial_i\psi(X^{0,\mathsf{flip}}_{\mathcal{N}(i)}))\right)\\
        &+p\sigma(2\partial_i\psi(X^{0}_{\mathcal{N}(i)}))^2+(1-p)\sigma(2\partial_i\psi(X^{0,\mathsf{flip}}_{\mathcal{N}(i)}))^2\bigg\vert X^0_{\mathcal{N}(i)},\mathcal{B}^{(\tau_{\ell})},\mathcal{G}_{\tau_{\ell}}\bigg]\\
        &=\mathbb{E}\left[\left(1-p(Y_2^{(\tau_{\ell})})\right)\left(\sigma(2\partial_i\psi(X^{0}_{\mathcal{N}(i)}))-\sigma(2\partial_i\psi(X^{0,\mathsf{flip}}_{\mathcal{N}(i)}))\right)^2\bigg\vert X^0_{\mathcal{N}(i)},\mathcal{B}^{(\tau_{\ell})},\mathcal{G}_{\tau_{\ell}}\right].
    \end{align*}
    This already certifies \Cref{eq:weak_lb} on $(\mathcal{A}^{(\tau_{\ell})})^c$ since this expectation is nonnegative for any initial configuration, and the occurrence of $\mathcal{A}^{(\tau_{\ell})}$ is determined only be $X^0$.

    To certify \Cref{eq:first_lb}, suppose that $\mathcal{A}^{(\tau_{\ell})}$ holds, which is completely determined by $X^0_{\mathcal{N}(i)}$. By \Cref{fact:prob_lb}, $1-p(Y_2^{(\tau_{\ell})})\geq \exp(-2\lambda)/2$ surely since the conditional probability of site $j$ taking any sign at each update time in $I_{2,2}^{(\tau_{\ell})}$ is lower bounded by $\exp(-2\lambda)/2$ conditional on any configuration. Therefore, we obtain the lower bound
    \begin{equation*}
        \mathbb{E}\left[Z^{(\tau_{\ell})}\bigg\vert X^0_{\mathcal{N}(i)},\mathcal{B}^{(\tau_{\ell})},\mathcal{G}_{\tau_{\ell}}\right]\geq \frac{\exp(-2\lambda)}{2}\cdot \mathbb{E}\left[\left(\sigma(2\partial_i\psi(X^{0}_{\mathcal{N}(i)}))-\sigma(2\partial_i\psi(X^{0,\mathsf{flip}}_{\mathcal{N}(i)}))\right)^2\bigg\vert X^0_{\mathcal{N}(i)},\mathcal{B}^{(\tau_{\ell})},\mathcal{G}_{\tau_{\ell}}\right].
    \end{equation*}

    \Cref{fact:km_sigmoid} implies that deterministically, we have the lower bound
    \begin{align*}
        \left(\sigma(2\partial_i\psi(X^{0}_{\mathcal{N}(i)})-\sigma(2\partial_i\psi(X^{0,\mathsf{flip}}_{\mathcal{N}(i)}))\right)^2&\geq \frac{\exp(-4\lambda)}{16}\left(\partial_i\psi(X^{0}_{\mathcal{N}(i)})-\partial_i\psi(X^{0,\mathsf{flip}}_{\mathcal{N}(i)})\right)^2.
    \end{align*}
    Finally, for any multilinear function, 
    \begin{equation*}
        \left(\partial_i\psi(X^{0}_{\mathcal{N}(i)})-\partial_i\psi(X^{0,\mathsf{flip}}_{\mathcal{N}(i)})\right)^2=4\partial_i\partial_j\psi(X^{0}_{\mathcal{N}(i)})^2,
    \end{equation*}
    and since this mixed partial derivative exceeds $\alpha$ in absolute value when $X^0$ satisfies $\mathcal{A}^{(\tau_{\ell})}$, we may conclude that
    \begin{equation*}
        \mathbb{E}\left[Z^{(\tau_{\ell})}\vert  \mathcal{A}^{(\tau_{\ell})}\cap \mathcal{B}^{(\tau_{\ell})},\mathcal{G}_{\tau_{\ell}}\right]\geq \frac{\alpha^2}{8}\exp(-6\lambda).
    \end{equation*}
    The inequality \Cref{eq:overall_lb} then follows from \Cref{eq:first_lb} and \Cref{eq:weak_lb} by the fact that $\mathcal{A}^{(\tau_{\ell})}$ is conditionally independent of $\mathcal{B}^{(\tau_{\ell})}$ given $\mathcal{G}_{\tau_{\ell}}$ since the latter event depends only on updated indices in $I^{(\tau_{\ell})}_2$, which are independent of all previous updates. It follows from \Cref{lem:burn_in} that
    \begin{equation*}
        \mathbb{E}[Z^{(\tau_{\ell})}\vert \mathcal{B}^{(\tau_{\ell})},\mathcal{G}_{\tau_{\ell}}]\geq \Pr\left(\mathcal{A}^{(\tau_{\ell})}\bigg\vert \mathcal{G}_{\tau_{\ell}}\right)\cdot \mathbb{E}\left[Z^{(\tau_{\ell})}\vert  \mathcal{A}^{(\tau_{\ell})}\cap \mathcal{B}^{(\tau_{\ell})},\mathcal{G}_{\tau_{\ell}}\right]\geq \frac{q_{\ref{lem:burn_in}}\alpha^2}{8}\exp(-6\lambda).\qedhere
    \end{equation*}
\end{proof}
\begin{remark}
    Note that using $Y_1^{(\tau_{\ell})}$, rather than $Y_2^{(\tau_{\ell})}$ in the cross-term of $Z^{(\tau_{\ell})}$ appears necessary. Indeed, doing any sort of weak conditioning on whether or not site $j$ flips in the middle interval induces nontrivial biases in the value of site $i$ at the end of $I^{(\tau_{\ell})}_{2,1}$, which may very well be $Y_2^{(\tau_{\ell})}$. To our knowledge, this same issue appears to arise in the analysis of Bresler, Gamarnik, and Shah \cite{DBLP:journals/tit/BreslerGS18}. We circumvent this issue by using the fact that the dependence between the random variables in our cross-term is broken by the existence of the second site update $Y_2^{(\tau_{\ell})}$ under $\mathcal{B}^{(\tau_{\ell})}$.
\end{remark}

The above bound has the following easy consequence:

\begin{corollary}
\label{corr:nb_full}
    Suppose that $i\sim j$ under \Cref{assumption:psi}. Then for any $\ell\geq 1$,
    \begin{equation}
    \label{eq:good_lb}
        \mathbb{E}[Z^{(\tau_{\ell})}\vert \mathcal{G}_{\tau_{\ell}}]\geq \frac{\alpha^2q_{\ref{lem:burn_in}}}{8}\exp(-2dL)\exp(-6\lambda)-2(1-\exp(-2dL)).
    \end{equation}
\end{corollary}
\begin{proof}
    We simply write
    \begin{align*}
        \mathbb{E}[Z^{(\tau_{\ell})}\vert \mathcal{G}_{\tau_{\ell}}]&=\Pr(\mathcal{B}^{(\tau_{\ell})}\vert \mathcal{G}_{\tau_{\ell}})\mathbb{E}[Z\vert \mathcal{B}^{(\tau_{\ell})},\mathcal{G}_{\tau_{\ell}}]+\Pr((\mathcal{B}^{\tau_{\ell}})^c\vert \mathcal{G}_{\tau_{\ell}})\mathbb{E}[Z^{(\tau_{\ell})}\vert (\mathcal{B}^{(\tau_{\ell})})^c,\mathcal{G}_{\tau_{\ell}}]\\
        &\geq \Pr(\mathcal{B}^{(\tau_{\ell})}\vert \mathcal{G}_{\tau_{\ell}})\frac{\alpha^2q_{\ref{lem:burn_in}}}{8}\exp(-6\lambda)-2(1-\Pr(\mathcal{B}^{(\tau_{\ell})}\vert \mathcal{G}_{\tau_{\ell}})),
    \end{align*}
    where we use the fact $\mathcal{B}^{(\tau_{\ell})}$ is independent of $\mathcal{G}_{\tau_{\ell}}$ along with \Cref{prop:nbr_lb} and the fact $\vert Z\vert\leq 2$ surely. We can then conclude via \Cref{lem:no_nbr_lb}.
\end{proof}

We now turn to bounding the expected value of the statistic in the case that $i\not\sim j$ to establish a quantitative separation in these cases.

\begin{lemma}[Non-Neighbor Upper Bound]
\label{lem:nnb_ub}
    Suppose that $i\not\sim j$ and that \Cref{assumption:psi} holds. Then it holds that
    \begin{equation}
    \label{eq:non_nbr_ub}
        \mathbb{E}[Z^{(\tau_{\ell})}\vert \mathcal{G}_{\tau_{\ell}}]\leq 2(1-\exp(-2dL)).
    \end{equation}
\end{lemma}
\begin{proof}
    First, note that
    \begin{equation*}
        \mathbb{E}[Z^{(\tau_{\ell})}\vert \mathcal{B}^{(\tau_{\ell})}, \mathcal{G}_{\tau_{\ell}}]=0.
    \end{equation*}
    Indeed, if $\mathcal{B}^{(\tau_{\ell})}$ occurs, then all each of $Y^{(\tau_{\ell})}_1,Y^{(\tau_{\ell})}_2,Y^{(\tau_{\ell})'}_1,Y^{(\tau_{\ell})'}_2$ are independent Bernoulli random variables with the same bias $p=p(X^0)$ given the starting configuration $X^0$ on $I_{2}^{(\tau_{\ell})}$ since no neighbor of $i$ ever gets updated on $\mathcal{B}^{(\tau_{\ell})}$ (recall $i\not\sim j$ by assumption). Therefore, by independence and linearity of expectation,
    \begin{align*}
        \mathbb{E}[Z^{(\tau_{\ell})}\vert X^0,\mathcal{B}^{(\tau_{\ell})}, \mathcal{G}_{\tau_{\ell}}]&=p(X^0)^2-2p(X^0)^2-p(X^0)^2\\
        &=0.
    \end{align*}
    Therefore, 
    \begin{equation*}
        \mathbb{E}[Z^{(\tau_{\ell})}\vert \mathcal{G}_{\tau_{\ell}}]=(1-\Pr(\mathcal{B}^{(\tau_{\ell})}\vert \mathcal{G}_{\tau_{\ell}}))\mathbb{E}[Z^{(\tau_{\ell})}\vert(\mathcal{B}^{(\tau_{\ell})})^c, \mathcal{G}_{\tau_{\ell}}]\leq 2(1-\exp(-2dL)),
    \end{equation*}
    since $\vert Z^{(\tau_{\ell})}\vert\leq 2$ surely and using \Cref{lem:no_nbr_lb}.
\end{proof}

We now set parameters to give a quantitative separation between the cases $i\sim j$ and $i\not\sim j$. Set 
\begin{equation}
\label{eq:L_val_2}
    L:=\frac{\alpha^2q_{\ref{lem:burn_in}}\exp(-6\lambda)}{64d}.
\end{equation}
We claim with this choice of $L$, the following holds:
\begin{gather*}
    \exp(-2dL)\geq 7/8\\
    2(1-\exp(-2dL))\leq \frac{1}{16}\alpha^2q_{\ref{lem:burn_in}}\exp(-6\lambda).
\end{gather*}
To see the first inequality, simply observe that $q_{\ref{lem:burn_in}}\leq 1/2$ and $\alpha^2\exp(-6\lambda)\leq \lambda^2\exp(-6\lambda)\leq 1$ as can easily be verified by analyzing the function $x\mapsto x^2\exp(-6x)$. It follows that $2dL\leq 1/8$, and thus $\exp(-2dL)\geq \exp(-1/8)\geq 7/8$. The second inequality holds since
\begin{equation*}
    2(1-\exp(-2dL))\leq 4dL
\end{equation*}
and then simple algebra gives the claim. It follows from \Cref{corr:nb_full} that if $i\sim j$, 
\begin{equation}
\label{eq:exp_lb}
    \mathbb{E}[Z^{(\tau_{\ell})}\vert \mathcal{G}_{\tau_{\ell}}]\geq \frac{7\alpha^2 q_{\ref{lem:burn_in}}}{64}\exp(-6\lambda),
\end{equation}
while if $i\not\sim j$, \Cref{lem:nnb_ub} implies
\begin{equation}
\label{eq:exp_ub}
    \mathbb{E}[Z^{(\tau_{\ell})}\vert \mathcal{G}_{\tau_{\ell}}]\leq \frac{\alpha^2 q_{\ref{lem:burn_in}}}{16}\exp(-6\lambda).
\end{equation}
We now define the threshold 
\begin{equation}
\label{eq:kappa_2}
\kappa=\frac{5\alpha^2 q_{\ref{lem:burn_in}}}{64}\exp(-6\lambda).
\end{equation}

\begin{theorem}
\label{thm:good_estimates_2}
    Suppose that \Cref{assumption:psi} holds and let $\delta>0$. Define $M$ by
    \begin{equation*}
        M \triangleq \frac{2000\log(2n^2/\delta)}{\kappa^2}.
    \end{equation*}
    Then for any $i\neq j\in [n]$, the following holds:
    \begin{enumerate}
        \item If $i\sim j$, then 
        \begin{equation*}
            \Pr\left(\frac{1}{M}\sum_{\ell=1}^M Z^{(\tau_{\ell})}\leq \kappa \right)\leq \delta/2n^2.
        \end{equation*}
        \item If $i\not \sim j$, then
        \begin{equation*}
            \Pr\left(\frac{1}{M}\sum_{\ell=1}^M Z^{(\tau_{\ell})}\geq \kappa \right)\leq \delta/2n^2
        \end{equation*}
    \end{enumerate}
\end{theorem}
\begin{proof}
    First suppose that $i\sim j$. By \Cref{eq:exp_lb}, the random process for $m=1,\ldots$
    \begin{equation*}
        \sum_{\ell=1}^m Z^{(\tau_{\ell})}-(7/5)m\kappa
    \end{equation*}
    is a submartingale adapted to the filtration $\mathcal{G}_{\tau_{m+1}}$. Note that each summand lies in $[-3,2]$. Therefore, by the
     Azuma-Hoeffding inequality,
    \begin{align*}
        \Pr\left(\frac{1}{M}\sum_{\ell=1}^M Z^{(\tau_{\ell})}\leq \kappa \right)&=\Pr\left(\frac{1}{M}\sum_{\ell=1}^M \left(Z^{(\tau_{\ell})}-(7/5)\kappa\right)\leq -2\kappa/5 \right)\\
        &\leq \exp\left(\frac{-M\kappa^2}{2000}\right)\\
        &\leq \frac{\delta}{2n^2},
    \end{align*}
    by our choice of $M$. A completely analogous argument for the case $i\not\sim j$ using the fact that $\sum_{\ell=1}^m (Z^{(\tau_{\ell})}-4\kappa/5)$ is a supermartingale adapted to the same filtration that must exceed $M\kappa/5$ on the desired event gives the second bound.
\end{proof}

Finally, we show that if $T$ is sufficiently large, then it is likely that we obtain sufficient samples for all pairs $(i,j)\in [n]^2$.
\begin{proposition}
\label{prop:enough_stops}
    Let $M$ be defined as in \Cref{thm:good_estimates_2}. For any fixed $i\neq j$, we have
    \begin{equation*}
        \Pr\left(\tau_M\geq Mr+2M/q_{\ref{lem:good_event_lb}}\right)\leq \frac{\delta}{2n^2}.
    \end{equation*}
\end{proposition}
\begin{proof}
    We first rewrite
    \begin{equation*}
\tau_M=\sum_{\ell=1}^M \tau_{\ell}-\tau_{\ell-1},
    \end{equation*}
    where we define $\tau_0=0$. By \Cref{lem:good_event_lb}, for any $\ell\geq 1$, it follows that the random variable $\tau_{\ell+1}-\tau_{\ell}$ is stochastically dominated by $r+G_{\ell}$ for an independent geometric random variable $G_{\ell}$ with parameter $q_{\ref{lem:good_event_lb}}$. Therefore, we have the stochastic domination of $\tau_M$ by the random variable $Mr+\sum_{\ell=1}^M G_{\ell}$ for independent geometric random variables $G_{\ell}$ with parameter $q_{\ref{lem:good_event_lb}}$. We find that
    \begin{equation*}
        \Pr\left(\tau_M\geq Mr+2M/q_{\ref{lem:good_event_lb}}\right)\leq \Pr\left(\sum_{\ell=1}^M G_{\ell}\geq 2M/q_{\ref{lem:good_event_lb}}\right).
    \end{equation*}
    Using the standard coupling between geometric random variables and Bernoulli random variables, this latter event is equivalent to the event that a Binomial random variable with $2M/q_{\ref{lem:good_event_lb}}$ trials and success probability $q_{\ref{lem:good_event_lb}}$ has at most $M$ successes. We thus have
    \begin{equation*}
        \Pr\left(\sum_{\ell=1}^M G_{\ell}\geq 2M/q_{\ref{lem:good_event_lb}}\right)=\Pr\left(\mathsf{Bin}(2M/q_{\ref{lem:good_event_lb}},q_{\ref{lem:good_event_lb}})\leq \mu/2\right),
    \end{equation*}
    where here $\mu=2M$ is the expected value of this binomial random variable. By the multiplicative Chernoff bound,
    \begin{equation*}
        \Pr\left(\mathsf{Bin}(2M/q_{\ref{lem:good_event_lb}},q_{\ref{lem:good_event_lb}})\leq \mu/2\right)\leq \exp\left(-\frac{\mu}{8}\right)=\exp\left(-\frac{M}{4}\right).
    \end{equation*}
    Since $\frac{M}{4}\geq \log(2n^2/\delta)$ by construction, the claim follows.
\end{proof}

\subsection{Final Algorithm and Guarantees}
\label{sec:alg_final}
With these results in order, we can state our final algorithm, \Cref{alg:markov_blanket}, and prove the correctness and runtime bounds.

\begin{algorithm}[H]
    \caption{$E=$ FindMarkovBlanket$(k,d,\alpha,\lambda,\delta)$}
    \label{alg:markov_blanket}
    \LinesNumbered

    Set
    \begin{gather*}
        q_{\ref{lem:burn_in}} = \frac{1}{2}\left(\frac{\exp(-2\lambda)}{2d}\right)^{k-2},\quad
        L = \frac{\alpha^2q_{\ref{lem:burn_in}}\exp(-6\lambda)}{64d},\quad
        \kappa = \frac{5\alpha^2 q_{\ref{lem:burn_in}}}{64}\exp(-6\lambda),\quad     q_{\ref{lem:good_event_lb}}=\frac{L^5}{6^5}\\
        M=\frac{2000\log(2n^2/\delta)}{\kappa^2},\quad
        r=\left\lceil \frac{\log(2\max\{1,2(k-2)\})}{L}\right\rceil,\quad
        T = L\cdot (Mr+2M/q_{\ref{lem:good_event_lb}})
    \end{gather*}

    Observe random process $(X_t)_{t=0}^T$ and $\Pi_k(T)$ for all $k\in [n]$.

    \For{$i<j\in [n]$}
        {Compute all stopping times $\tau_1,\ldots,\tau_{f_{ij}}\leq T$ as in \Cref{eq:stopping_times} where $f_{ij}=\max\{\ell: \tau_{\ell}\leq T\}$.\\
        \If{$f_{ij}<M$}{Return $\perp$. \tcp*{algorithm fails}}
        \Else{Add $(i,j)$ to $E$ if  $
            \frac{1}{M}\sum_{\ell=1}^M Z^{(\tau_{\ell})}\geq \kappa$, where $Z^{(\tau_{\ell})}$ is defined as in \Cref{eq:z_def}.}

        }
    
\end{algorithm}

\begin{theorem}
\label{thm:structure_recovery}
    Given $\delta>0$ and the trajectory of Glauber dynamics from a $(k,d,\alpha,\lambda)$-MRF $\mu$ of length $T=O_{k,d,\alpha,\lambda}(\log(n/\delta))$, \Cref{alg:markov_blanket} returns a graph $G$ such that $(i,j)\in G$ if and only if $i\sim j$ in $\mu$ with probability at least $1-\delta$. The runtime of \Cref{alg:markov_blanket} is $O_{k,d,\alpha,\lambda}(n^2\log (n/\delta))$.
\end{theorem}
\begin{proof}
    We assume that $\Pi_i(T)$ is given as an ordered list for each $i\in [n]$, which has length $O_{k,d,\alpha,\lambda}(\log(n/\delta))$ with probability $1-\delta$ for all $i\in [n]$ simultaneously by standard concentration bounds. In particular, given $i\neq j\in [n]$, one can compute the stopping times in time $O_{k,d,\alpha,\lambda}(\log(n/\delta))$ by a linear scan. Moreover, one can compute the statistic $\frac{1}{M}\sum_{\ell=1}^M Z^{(\tau_{\ell})}$ in $O_{k,d,\alpha,\lambda}(\log(n/\delta))$ time. Therefore, the runtime is indeed $O_{k,d,\alpha,\lambda}(n^2\log(n/\delta))$ since this procedure is done for all $i<j$.

    Correctness follows immediately from \Cref{thm:good_estimates_2} and \Cref{prop:enough_stops}; the former result says that the probability that the statistic errs in correctly outputting whether $i\sim j$ after the first $M$ stopping times for all pairs $i<j$ is at most $\delta/2$ by a union bound. Moreover, the latter implies that the probability of failing to have at most $M$ stopping times on a trajectory of length $T$ for any pair $i<j$ is also at most $\delta/2$ (note that each interval in \Cref{thm:good_estimates_2} is of length $L$). Taking a further union bound gives the desired error probability after replacing $\delta$ by $\delta/2$.
\end{proof}

\subsection{Structure Learning with Unobserved Variables}
\label{sec:unobserved}

\Cref{thm:structure_recovery} provides a stark computational benefit to structure learning in bounded-degree MRFs that overcomes the SPN barrier for i.i.d. samples. However, a simple inspection of our results shows that dynamical samples provide an \emph{information-theoretic} benefit over the i.i.d. case in the presence of unobserved variables.

First, we remark that SPN examples certify that if one only observes the values of a set $S$ of sites, it can be information-theoretically impossible to determine the induced dependency graph $G[S]$ from i.i.d. samples from $\mu$ when marginalizing out the remaining variables. The following simple result is from Bresler, Gamarnik, and Shah~\cite{DBLP:conf/nips/BreslerGS14a} as well as Klivans and Meka~\cite{DBLP:conf/focs/KlivansM17}:

\begin{lemma}
\label{lem:marginalized_imp}
    Let $\psi_1(\bm{x})=\prod_{i=1}^n x_i$ and $\psi_2(\bm{x})=1$. Then the law of $(X_1,\ldots,X_{n-1})$ is $\mathsf{Unif}(\{-1,1\}^{n-1})$ under $\mu_{\psi_1}$ and $\mu_{\psi_2}$.  In particular, given arbitrarily many i.i.d. samples of $(X_1,\ldots,X_{n-1})$ under either $\mu_{\psi_1}$ or $\mu_{\psi_2}$ when $X_n$ is unobserved, it is information-theoretically impossible to determine whether the samples come from $\mu_{\psi_1}$ or $\mu_{\psi_2}$.
\end{lemma}

As a consequence, it follows that given a MRF $\mu$ with dependency graph $G$, and given i.i.d samples of $(X_1,\ldots,X_{n-1})$ from $\mu$ after marginalizing out $X_n$, one cannot learn the induced dependency graph $G'=G[\{1,\ldots,n-1\}]$ on the observed variables $\{1,\ldots,n-1\}$.

However, it is easy to see that under dynamical samples, if the MRF $\mu$ satisfies \Cref{assumption:psi}, one can nonetheless recover the induced subgraph of the observed variables using \Cref{alg:markov_blanket}:

\begin{corollary}
    Let $\mu$ be a $(k,d,\alpha,\lambda)$-MRF with minimal dependency graph $G$ and let $S\subseteq [n]$ be a subset of observed sites. Then, given $\delta>0$ and the trajectory of continuous-time Glauber dynamics from $\mu$ of length $T=O_{k,d,\alpha,\lambda}(\log(n/\delta))$ but restricted to only include updates of the observed sites $S$, \Cref{alg:markov_blanket} returns a graph $G'$ on $S$ such that $(i,j)\in G'$ if and only if $i\sim j$ in $G$ with probability at least $1-\delta$, for all $i,j\in S$. The runtime of \Cref{alg:markov_blanket} is $O_{k,d,\alpha,\lambda}(n^2\log (n/\delta))$.
\end{corollary}
\begin{proof}[Proof Sketch]
    The key observation is that for any pair $i\neq j\in S$, the requisite statistics are all still measurable with respect to the observed trajectory restricted to sites in $S$. Moreover, the requisite conditional probabilities remain valid for each such pair since we assume the overall MRF $\mu$ satisfies \Cref{assumption:psi}; therefore, the estimates remain valid even though we cannot observe update times for nodes outside $S$. Therefore, the statistic still distinguishes whether or not $i\sim j$ in $G=G(\mu)$ for each pair $i\neq j\in S$.
\end{proof}

In particular, \Cref{alg:markov_blanket} enables one to determine the dependency structure across all observed sites even in the presence of unobserved nodes from dynamics. Given the impossibility of doing so from i.i.d. samples from \Cref{lem:marginalized_imp}, it follows that correlated samples also overcome information-theoretic barriers to determining MRF dependencies with latent sites.

\subsection{Lower Bounds}
\label{sec:app_lbs}

In this section, we provide evidence that the simple structure learning algorithm we obtained is, in a slightly idealized setting, the simplest possible statistic that can perform structure learning. To state the result, we need the following definition:

\begin{definition}[$(M,T)$-local Dependency Tests]
\label{defn:local_test}
    Given $i\neq j\in [n]$, an $(M,T)$-local dependency test for $(i,j)$ is defined as follows. Consider the following sampling procedure: for each $m=1,\ldots,M$, sample $X^{(m)}_{-\{i,j\}}\in \{-1,1\}^{n-2}$ according to the stationary measure $\mu$. Then, the test receives $T$ independent samples of $X^{(m)}_i$ conditioned on $X^{(m)}_{-\{i,j\}}$ and $X_{j}^{(m)}=\varepsilon$ for each $\varepsilon\in \{-1,1\}$, which we denote via the vector $Y_m:=((X^{(m),1}_{i,+},\ldots,X^{(m),T}_{i,+}),(X^{(m),1}_{i,-},\ldots,X^{(m),T}_{i,-}))$ in the natural way. The dependency test then returns $i\sim j$ or $i\not\sim j$ as a function of $(Y_1,\ldots,Y_M)\in (\{-1,1\}^{2T})^M$.
\end{definition}

To understand the definition, ``locality" refers to the fact that the test to determine whether $i\sim j$ depends \emph{only} on observations of samples of $X_i$ where the sign of $X_j$ varies while $X_{\mathcal{N}(i)\setminus \{j\}}$ remains fixed and drawn from the stationary measure. In particular, our algorithm can be viewed as a $(M,2)$-local test while the BGS algorithm for the Ising model can be viewed as a $(M,1)$-local test for suitable values of $M\in \mathbb{N}$ depending on the model parameters as in \Cref{assumption:psi} and desired failure probability.

Note that this model can be viewed as an abstraction of both our test and that of Bresler, Gamarnik, and Shah \cite{DBLP:journals/tit/BreslerGS18} for the Ising setting that avoids reasoning about nice events during the complex evolution of Glauber dynamics. However, for sufficiently small width $\lambda<1$, the Glauber dynamics rapidly mix to the stationary measure by the Dobrushin condition~\cite{dobrushin1987completely}. Since the event that $i$ and $j$ both update multiple times in a small window is rather atypical and thus typically occur with linear spacing, many of the Glauber observations on these observable events will actually be close to this idealized sampling model. 

With this definition, we may now state an impossibility result: there is no $(M,1)$-local dependency test for $k$-MRFs for any $k\geq 3$ with nontrivial failure probability. 

\begin{theorem}
\label{thm:local_lb}
    For any fixed $\alpha>0$, there is $\lambda=\lambda(\alpha)>0$ sufficiently large such that there exists a pair of $(3,3,\alpha,\lambda)$-MRFs $\mu_1$ and $\mu_2$ with Hamiltonians $\psi_1$ and $\psi_2$ such that $1\sim 2$ in $\mu_{1}$, $1\not\sim 2$ in $\mu_{2}$, but the distribution of the output of any $(M,1)$-local dependency algorithm evaluated on the pair $(1,2)$ is the same in $\mu_{1}$ and $\mu_{2}$.
\end{theorem}
We defer the proof of \Cref{thm:local_lb} to \Cref{sec:app_lb}. In particular, it is information-theoretically impossible for any $(M,1)$-local dependency test to succeed in determining if $i\sim j$ for all 3-MRFs, even when $M=\infty$. \Cref{thm:local_lb} stands in contrast to the Ising case since the BGS algorithm (or at least a modification thereof) only requires a single conditional sample of $X_i$ for each value of $X_j$. In this slightly idealized model, any algorithm for learning the underlying structure of the MRF necessarily requires more conditional samples for each setting of $X_j=\pm 1$, or must combine samples for different $(i,j)$ pairs in more complex ways (i.e. must be non-local). Since our algorithm is $(M,2)$-local for $M=O(\log n)$, \Cref{thm:local_lb} suggests that our approach is essentially the simplest possible statistic for recovering the dependency structure from dynamical observations.

\section{Parameter Learning via Logistic Regression}
\label{sec:app_lr}
In this section, we show how, given the Markov blanket from the previous section, we may then employ logistic regression to recover the actual interactions in a further $O_{\alpha,d,\lambda,k}(n)$ time. The key points are that (i) we only need to find coefficients for a known set of at most $d^{k-1}$ terms for each node, rather than $n^{k-1}$ when the Markov blanket is unknown, and (ii) the correlated samples from the dynamics are still informative enough for logistic regression to succeed in parameter recovery, not just structure recovery. We provide the relevant definitions for logistic regression in \Cref{sec:lr}. We then provide uniform convergence bounds in \Cref{sec:uc} and unbiasedness estimates in \Cref{sec:ub}. Finally, we put together these results in \Cref{sec:lr_guarantees} to establish efficient parameter recovery via logistic regression for bounded-degree MRFs given the dependency graph $G$.

For this section, we will consider the discrete-time Glauber dynamics:
\begin{definition}
\label{defn:discrete}
    Given a MRF $\mu=\mu_{\psi}$, the discrete-time Glauber dynamics is defined as follows:
    \begin{enumerate}
    \item The spin system starts at an arbitrary initial configuration $X^0$.
    \item For each $t\geq 1$, the following holds: first, sample a uniform index $i_t\in [n]$. Then $X^{t}$ is defined by $X^{t}_{-i_t}=X^{t-1}_{-i_t}$ while $X^{t}_{i_t}$ is resampled according to the distribution satisfying
    \begin{equation}
        \Pr\left(X^{t}_{i_t}=1\vert X^t_{-i_t}\right)=\sigma(2\partial_{i_t}\psi(X^t_{-i_t})).
    \end{equation}
\end{enumerate} 
\end{definition}

In this section, we will now write $\mathcal{F}_{t}=\sigma(X^0,\ldots,X^t,i_1,\ldots,i_t)$ for the filtration induced by the discrete-time process up to time $t$. 

\subsection{Logistic Regression}
\label{sec:lr}
We now describe the logistic regression formulation for the node $n$ using Glauber trajectories; the procedure is identical for all $i\in [n]$ with minor notational changes. 

\begin{definition}[Generated Samples for Node $n$]
    Given a trajectory of Glauber dynamics with arbitrary initial starting configuration $X^0$ as in \Cref{defn:discrete}, let $\tau_1<\tau_2<\ldots<\tau_M$ be a sequence of $\mathcal{F}_t$-stopping times with the property that $i_{\tau_{\ell}}=n$ almost surely for each $\ell\geq 1$.\footnote{We will specify the precise stopping rule we will use for our parameter recovery results shortly.} We then define the generated samples for node $n$ by
\begin{equation*}
    Z^{\ell} = X^{\tau_{\ell}}=(X^{\tau_{\ell}}_{-n},X^{\tau_{\ell}}_n).
\end{equation*}
\end{definition}

Define the logistic loss by $\ell(z)=\log(1+\exp(-z))$. It is well-known that $\ell$ is 1-Lipschitz and convex.
We will assume that the neighborhood $\mathcal{N}(i)$ is known (i.e. via the algorithmic results of the previous section). Under \Cref{assumption:psi}, $\vert \mathcal{N}(n)\vert=d_n\leq d$. We then have the following definition:
\begin{definition}
\label{defn:log_reg}
    Given parameters $d_n,k\in \mathbb{N}$ and $\lambda>0$, let $P(d_n,k,\lambda)$ denote the set of polynomials $p:\{-1,1\}^{d_n}\to \mathbb{R}$ of degree at most $k$, represented as a vector of coefficients in some order, satisfying $\|p\|_1\leq \lambda$. Then the \textbf{logistic regression problem} for node $n$ is defined via
\begin{equation}
\label{eq:emp_opt}
    \widehat{p}\triangleq \arg\min_{p\in P(d_n,k,\lambda)} \frac{1}{M}\sum_{\ell=1}^M \ell\left(2\cdot X^{\tau_{\ell}}_n\cdot p\left(X_{\mathcal{N}(n)}^{\tau_{\ell}}\right)\right).
\end{equation}
Note that this is a convex program over the at most $O(d^k)$ coefficients in the representation of polynomials in $P(d_n,k,\lambda)$.
\end{definition}

\subsection{Uniform Convergence of Logistic Losses}
\label{sec:uc}

In this section, we establish the convergence of empirical logistic losses for any sequence of stopping times where node $n$ updates to the corresponding population logistic losses. Note that these population logistic losses are \emph{random}, as the law of the marginal on sites outside site $n$ will depend on the configuration at the previous stopping times. The main result of this section is the following uniform martingale concentration bound of logistic losses for \emph{any} stochastic process:

\begin{theorem}
\label{thm:uniform_convergence}
    There exists an absolute constant $C>0$ such that the following holds. Let $\lambda>0$ and $k\leq d\in \mathbb{N}$. Let $Z^1=(X^1,Y^1),\ldots,Z^T=(X^T,Y^T)\in \{-1,1\}^d\times \{-1,1\}$ be any stochastic process. For any $u>0$, with all but probability $\exp\left(-u^2/C\lambda^2T\right)$, it holds that
    \begin{align*}
        \sup_{p\in P(d,k,\lambda)} \sum_{\ell=1}^T\mathbb{E}\left[\ell(2\cdot Y^{\ell}\cdot p(X^{\ell}))\vert Z^1,\ldots,Z^{\ell-1}\right]-\sum_{\ell=1}^T\ell(2\cdot Y^{\ell}\cdot p(X^{\ell}))\leq C\lambda\log^{3/2}(T)\sqrt{Tk\log(d)}+u.
    \end{align*}
\end{theorem}

\begin{proof}[Proof Sketch]
    The argument is essentially identical to \cite[Theorem 5.4]{unified} (and in fact can be formally deduced from their result), so we only highlight the relevant differences. As demonstrated by Gaitonde and Mossel, the desired tail bound can be proven by bounding the \emph{sequential Rademacher complexity} \cite{DBLP:journals/jmlr/RakhlinST15,DBLP:conf/colt/RakhlinS17} of the class of functions on $\{-1,1\}^d\times \{-1,1\}$ of the form $(\bm{x},y)\mapsto \ell(2yp(\bm{x}))$ where $p\in P(d,k,\lambda)$. After applying the Lipschitz contraction principle for sequential Rademacher complexity~\cite{DBLP:journals/jmlr/RakhlinST15}, it suffices to bound the sequential Rademacher complexity of the class of functions $(\bm{x},y)\mapsto yp(\bm{x})$ where $p\in P(d,k,\lambda)$. Formally, this is the quantity
    \begin{equation*}
\max_{(\bm{x}_1,y_1),\ldots,(\bm{x}_T,y_T)}\mathbb{E}_{\bm{\varepsilon}}\left[\sup_{p\in P(d,k,\lambda)}\sum_{t=1}^T \varepsilon_i y_t(\varepsilon_{1:t-1})p(\bm{x}(\varepsilon_{1:t-1}))\right],
    \end{equation*}
where the $\varepsilon_1,\ldots,\varepsilon_T$ are i.i.d. signs and $(\bm{x}_t,y_t)$ can depend on $\varepsilon_1,\ldots,\varepsilon_{t-1}$. We can rewrite this as
\begin{align*}
    \max_{(\bm{x}_1,y_1),\ldots,(\bm{x}_T,y_T)}&\mathbb{E}_{\bm{\varepsilon}}\left[\sup_{p\in P(d,k,\lambda)}\sum_{t=1}^T \langle p,\varepsilon_t y_t(\varepsilon_{1:t-1}) M_k(\bm{x}(\varepsilon_{1:t-1})\rangle\right]\\
    &=\lambda\max_{(\bm{x}_1,y_1),\ldots,(\bm{x}_T,y_T)}\mathbb{E}_{\bm{\varepsilon}}\left[\left\|\sum_{t=1}^T \varepsilon_t y_t(\varepsilon_{1:t-1}) M_k(\bm{x}(\varepsilon_{1:t-1}))\right\|_{\infty}\right],
\end{align*}
where $\langle\cdot,\cdot\rangle$ is the standard inner product, $M_k(\bm{x})$ is the vector of all monomials of degree at most $k$ of $\bm{x}$ (and thus is a vector with $\pm 1$ coordinates), and then applying $\ell_p$ duality. Note that each component of the vector $\sum_{t=1}^T \varepsilon_t y_t(\varepsilon_{1:t-1}) M_k(\bm{x}(\varepsilon_{1:t-1}))$ forms a standard random walk of length $T$, and it is well-known that the maximum of any $L$ such random walks (with arbitrary correlation) has expectation $O(\sqrt{T\log L})$. Here, the vector has dimension $O(d^k)$, leading to a bound of $O(\sqrt{Tk \log(d)})$. The rest of the argument is identical to~\cite[Theorem 5.4]{unified}.
\end{proof}

For any fixed $p^*\in P(d,k,\lambda)$, we can also easily obtain \emph{two-sided} concentration bounds as a simple consequence of the Azuma-Hoeffding inequality:

\begin{lemma}
\label{lem:fixed_good}
    For any fixed $p^*\in P(d,k,\lambda)$, any stochastic process $Z^1=(X^1,Y^1),\ldots,Z^T=(X^T,Y^T)\in \{-1,1\}^d\times \{-1,1\}$, and any $u>0$, it holds with all but probability $\exp(-u^2/32\lambda^2T)$ that
    \begin{equation*}
        \left\vert \sum_{\ell=1}^T\mathbb{E}\left[\ell(2\cdot Y^{\ell}\cdot p^*(X^{\ell}))\vert Z^1,\ldots,Z^{\ell-1}\right]-\sum_{\ell=1}^T\ell(2\cdot Y^{\ell}\cdot p^*(X^{\ell}))\right\vert\leq u.
    \end{equation*}
\end{lemma}
\begin{proof}
    Observe that for any $\ell\geq 1$, the martingale increments satisfy
    \begin{align*}
        \left\vert\mathbb{E}\left[\ell(2Y^{\ell} p^*(X^{\ell}))\vert Z^1,\ldots,Z^{\ell-1}\right]-\ell(2Y^{\ell}p^*(X^{\ell}))\right\vert&\leq \max_{(\bm{x}_1,y_1),(\bm{x}_2,y_2)} \vert \ell(2y_1p^*(\bm{x}_1))-\ell(2y_2p^*(\bm{x}_2))\vert\\
        &\leq 2\max_{(\bm{x}_1,y_1),(\bm{x}_2,y_2)} \vert y_1p^*(\bm{x}_1)-y_2p^*(\bm{x}_2)\vert\\
        &\leq 4\lambda,
    \end{align*}
    where in the last line, we use the assumption $\|p^*\|_1\leq \lambda$. Therefore, the claim is an immediate consequence of the Azuma-Hoeffding inequality.
\end{proof}

\subsection{Unbiasedness of Neighborhood}
\label{sec:ub}
To leverage the existing machinery of Klivans and Meka \cite{DBLP:conf/focs/KlivansM17} for parameter recovery, we first must establish $\delta$-unbiasedness for sites in $\mathcal{N}(n)$ after running Glauber dynamics for a relatively small interval. In this subsection, we prove our main auxiliary result showing that conditioned on the sequence of updated sites being fairly typical on a (possibly random) interval, any small subset of variables will form a $\delta$-unbiased distribution for an appropriate choice of $\delta$.

\begin{proposition}
\label{prop:nbd_unbiased}
    There exists an absolute constant $c>0$ such that the following holds. Suppose that $\mu$ is a $(k,d,\lambda)$-MRF. Let $X^0$ be an arbitrary configuration, $S\subseteq [n]$ be any subset, and $\ell\in \mathbb{N}$ be an arbitrary parameter. Let $\tau\geq \vert S\vert$ be any stopping time that is measurable with respect to the sequence of updated sites, and let $\mathcal{E}_{\ell,S,\tau}$ denote the event that during the first $\tau$ steps of Glauber dynamics starting at $X^0$, every site in $S$ is updated at least once, and for each $i\in S$, every site in $\mathcal{N}(i)$ is updated at most $\ell\geq 1$ times. Then conditional on $\mathcal{E}_{\ell,S,\tau}$, $X^{\tau}_{S}$ is a $\delta$-unbiased distribution on $\{-1,1\}^S$ for 
    \begin{equation*}
        \delta:=c\exp(-6\lambda k\ell).
    \end{equation*}
\end{proposition}
\begin{proof}
    We will prove the stronger claim that for any site $i\in S$, conditioned on \emph{any} update sequence satisfying $\mathcal{E}_{\ell,S,\tau}$ as well as \emph{any} fixing of the actual updates of the spins outside of the last update of site $i$, the likelihood ratio of $X^{\tau}_i$ remains suitably bounded. This implies the result by a simple averaging argument: let $\bm{x}\in \{-1,1\}^{S\setminus \{i\}}$ be an arbitrary spin configuration. It suffices to uniformly bound
    \begin{align*}
        \frac{\Pr\left(X^{\tau}_i=1\vert X^T_{S\setminus \{i\}}=\bm{x},\mathcal{E}_{\ell,S,\tau}\right)}{\Pr\left(X^{\tau}_i=-1\vert X^{\tau}_{S\setminus \{i\}}=\bm{x},\mathcal{E}_{\ell,S,\tau}\right)}&=\frac{\sum\limits_{(i_1,\ldots,i_{\tau})\in \mathcal{E}_{\ell,S,\tau}}\Pr\left(X^{\tau}_i=1\vert X^{\tau}_{S\setminus \{i\}}=\bm{x},i_1,\ldots,i_{\tau}\right)\Pr(i_1,\ldots,i_{\tau}\vert X^{\tau}_{S\setminus \{i\}}=\bm{x})}{\sum\limits_{(i_1,\ldots,i_{\tau})\in \mathcal{E}_{\ell,S,\tau}}\Pr\left(X^{\tau}_i=-1\vert X^{\tau}_{S\setminus \{i\}}=\bm{x},i_1,\ldots,i_{\tau}\right)\Pr(i_1,\ldots,i_{\tau}\vert X^{\tau}_{S\setminus \{i\}}=\bm{x})}\\
        &\leq \max_{(i_1,\ldots, i_{\tau})\in \mathcal{E}_{\ell,S,{\tau}}}\frac{\Pr\left(X^{\tau}_i=1\vert X^{\tau}_{S\setminus \{i\}}=\bm{x},i_1,\ldots,i_{\tau}\right)}{\Pr\left(X^{\tau}_i=-1\vert X^{\tau}_{S\setminus \{i\}}=\bm{x},i_1,\ldots,i_{\tau}\right)}\\
        &=\max_{(i_1,\ldots, i_{\tau})\in \mathcal{E}_{\ell,S,{\tau}}}\frac{\Pr\left(X^{\tau}_i=1, X^{\tau}_{S\setminus \{i\}}=\bm{x}\bigg\vert i_1,\ldots,i_{\tau}\right)}{\Pr\left(X^{\tau}_i=-1, X^{\tau}_{S\setminus \{i\}}=\bm{x}\bigg\vert i_1,\ldots,i_{\tau}\right)}
    \end{align*}
    Here, the sum is over all sequences of updates satisfying $\mathcal{E}_{\ell,S,{\tau}}$. Therefore, we may fix a sequence $(i_1,\ldots,i_{\tau})$ satisfying $\mathcal{E}_{\ell,S,{\tau}}$ and provide a uniform bound conditioned on this sequence. 

    Given $(i_1,\ldots,i_{\tau})$ satisfying $\mathcal{E}_{\ell,S,{\tau}}$, let $\tau_i\in [\tau]$ denote the \emph{last} update time of site $i$, which exists by definition. For fixed $\bm{x}\in \{-1,1\}^{\vert S\vert-1}$ and $(i_1,\ldots,i_{\tau})$ as above, let $\mathcal{P}(\bm{x},i_1,\ldots,i_{\tau})$ denote the set of all $(y_1,\ldots,y_{\tau_i-1},y_{\tau_i+1},\ldots,y_{\tau})\in \{-1,1\}^{T-1}$ such that setting $X^{t}_{i_t}$ to $y_t$ for $t\neq \tau_i$ satisfies $X^T_{S\setminus \{i\}}=\bm{x}$. In words, $\mathcal{P}(\bm{x},i_1,\ldots,i_{\tau})$ is the set of all update values consistent with obtaining $\bm{x}$ on $S\setminus \{i\}$ given the sequence of updates. An analogous argument shows that
    \begin{align*}
        \frac{\Pr\left(X^{\tau}_i=1, X^{\tau}_{S\setminus \{i\}}=\bm{x}\bigg\vert i_1,\ldots,i_{\tau}\right)}{\Pr\left(X^{\tau}_i=-1, X^{\tau}_{S\setminus \{i\}}=\bm{x}\bigg\vert i_1,\ldots,i_{\tau}\right)}
        &\leq \max_{\bm{y}\in \mathcal{P}(\bm{x},i_1,\ldots,i_{\tau})}\frac{\Pr\left(X^{\tau_i}_i=1, X^t_{i_t}=y_t \,\forall t\in [{\tau}]\setminus \{\tau_i\}\vert i_1,\ldots,i_{\tau}\right)}{\Pr\left(X^{\tau_i}_i=-1, X^t_{i_t}=y_t \,\forall t\in [{\tau}]\setminus \{\tau_i\}\vert i_1,\ldots,i_{\tau}\right)}\\
        &\leq \max_{\bm{y}\in \{-1,1\}^{{\tau}-1}}\frac{\Pr\left(X^{\tau_i}_i=1, X^t_{i_t}=y_t \,\forall t\in [{\tau}]\setminus \{\tau_i\}\vert i_1,\ldots,i_T\right)}{\Pr\left(X^{\tau_i}_i=-1, X^t_{i_t}=y_t \,\forall t\in [{\tau}]\setminus \{\tau_i\}\vert i_1,\ldots,i_{\tau}\right)}
    \end{align*}

    Therefore, it suffices to upper bound this latter ratio for any choice of $(i_1,\ldots,i_{\tau})$ satisfying $\mathcal{E}_{\ell,S,\tau}$ and any sequence of spins $\bm{y}$ as claimed. 
    
    However, this latter ratio is quite straightforward to write out. For a given configuration $X^t$, we write $X^{t,\pm}$ to denote the same configuration with the sign of $i$ is set to $\pm 1$. Since $\tau$ is measurable with respect to just the sequence of updated indices, the actual updates factorize according to \Cref{eq:glauber}, so we have
    \begin{align*}
        \frac{\Pr\left(X^{\tau_i}_i=1, X^t_{i_t}=y_t \,\forall t\in [{\tau}]\setminus \tau_i\vert i_1,\ldots,i_{\tau}\right)}{\Pr\left(X^{\tau_i}_i=-1, X^t_{i_t}=y_t \,\forall t\in [{\tau}]\setminus \tau_i\vert i_1,\ldots,i_{\tau}\right)}=\prod_{t=1}^{\tau_{i}-1}&\frac{\Pr(X^t_{i_t}=y_t\vert X^{t-1}_{-i_t})}{\Pr(X^t_{i_t}=y_t\vert X^{t-1}_{-i_t})}\cdot \frac{\Pr(X^{\tau_i}_{i}=1\vert X^{\tau_i-1}_i)}{\Pr(X^{\tau_i}_{i}=-1\vert X^{\tau_i-1}_i)}\\
        &\cdot \prod_{t=\tau_i+1}^{\tau} \frac{\Pr(X^t_{i_t}=y_t\vert X^{t-1,+}_{-i_t})}{\Pr(X^t_{i_t}=y_t\vert X^{t-1,-}_{-i_t})},
    \end{align*}
    where here, we implicitly write $X^t$ to denote the configuration at time $t$ where the updates have been according to $\bm{y}$. Clearly the first set of products cancels. By \Cref{fact:prob_lb}, the second ratio is at most $O(\exp(2\lambda))$. For the latter products, we observe that for any $t>\tau_i$, 
    \begin{align*}
        \frac{\Pr(X^t_{i_t}=y_t\vert X^{t-1}_{-i_t},X^t_i=1)}{\Pr(X^t_{i_t}=y_t\vert X^{t-1}_{-i_t},X^t_i=-1)}&=\frac{\sigma(2y_t\partial_{i_t}\psi(X_{-i_t}^{t,+}))}{\sigma(2y_t\partial_{i_t}\psi(X_{-i_t}^{t,-}))}\\
        &=\frac{1+\exp(-2y_t\partial_{i_t}\psi(X_{-i_t}^{t,-}))}{1+\exp(-2y_t\partial_{i_t}\psi(X_{-i_t}^{t,+}))}\\
        &\leq \exp(2\vert \partial_{i_t}\psi(X_{-i_t}^{t,-})-\partial_{i_t}\psi(X_{-i_t}^{t,-})\vert)\\
        &\leq \exp(4\vert \partial_{i_t}\partial_i \psi(X^t_{-i_t,i})\vert)\\
        &\leq \exp(4\|\partial_{i}\partial_{i_t}\psi\|_{1}).
    \end{align*}
    Therefore, we may bound last set of products by
    \begin{equation*}
        \exp\left(4\sum_{t=\tau_i+1}^{\tau} \|\partial_{i}\partial_{i_t}\psi\|_{1}\right)= \exp\left(4\sum_{j\neq i} N_j\|\partial_{i}\partial_{j}\psi\|_{1}\right),
    \end{equation*}
    where we write $N_j$ for the number of times $j$ appears in $i_{\tau_i+1},\ldots,i_{\tau}$. However, note that the mixed partial is zero unless $j\in \mathcal{N}(i)$; moreover, by the definition of $\mathcal{E}_{\ell,S,{\tau}}$, each such neighbor appears at most $\ell$ times. Thus, the ratio is bounded by
    \begin{equation*}
        \exp\left(4\ell\sum_{j\in \mathcal{N}(i)} \|\partial_{i}\partial_{j}\psi\|_{1}\right).
    \end{equation*}
    Finally, we note that
    \begin{equation*}
        \sum_{j\in \mathcal{N}(i)} \|\partial_{i}\partial_{j}\psi\|_{1}=\sum_{S\subseteq \mathcal{N}(i)} \vert S\vert\vert \widehat{\psi}(S\cup \{i\})\vert\leq k\|\psi\|_1,
    \end{equation*}
    since each monomial of $\psi$ is size at most $k$ so can belong to at most $k-1$ neighbors of $i$. Thus the likelihood ratio is finally bounded by 
    \begin{equation*}
        O(\exp(6\ell k\lambda)).
    \end{equation*}
   An identical argument applies for the reciprocal likelihood ratio, and so the claim follows.
\end{proof}
\begin{remark}
    Note that instead of $k\lambda$, the above argument gives a somewhat sharper bound of $\|L_i\psi\|_{1}$, where $L_i$ is the discrete Laplacian operator. This can give an improvement if one knows that most monomials of $\psi$ are typically less than degree $k$ under the spectral sample.
\end{remark}

\begin{remark}
    Note that the additional $k$ dependency in the exponent for the minimal conditional variance of a site can be necessary without leveraging more randomness in the update sequence. To see this, consider the MRF $\mu$ with Hamiltonian
    \begin{equation*}
        \psi(\bm{x})=x_1\ldots x_k.
    \end{equation*}
    Consider the sequential scan dynamics that updates the elements in sequential order $1,2,\ldots,d$ with initial starting configuration $X^0=\bm{1}$. We claim that given $X^d_j=1$ for each $j=2,\ldots,d$ at the end of these updates, the conditional variance of $X_1$ is indeed $\exp(-\Omega(k))$. A similar calculation to the above shows that
    \begin{align*}
        \frac{\Pr(X^1_1 = 1\vert X^d_{-i}=\bm{1})}{\Pr(X^1_1 = -1\vert X^d_{-i}=\bm{1})}&=\prod_{j=2}^d\frac{\Pr(X^j_j=1\vert X^j_{-i,j}=\bm{1},X^1_1=-1)}{\Pr(X^j_j=1\vert X^j_{-i,j}=\bm{1},X^1_1=1)}\\
        &=\exp\left(2(k-1)\right).
    \end{align*}
\end{remark}

We now apply \Cref{prop:nbd_unbiased} to the following sequence of stopping times. We will let $S=\mathcal{N}(n)$ and define stopping times as follows for $k\geq 1$:
\begin{gather*}
    \tau_0 = 0,\\
    \tau_{k+1} = \inf\{t\geq \tau_{k}+4\log(d)\cdot n: i_t=n\}.
\end{gather*}
In words, these stopping times are the sequence of first update times of node $n$ after waiting at least $4\log(d)\cdot n$ steps between stopping times. In a slight abuse of notation, we will write $\mathcal{E}_{\ell,S,\tau_k}$ for the event that on the interval $[\tau_{k-1}+1,\tau_k]$, every site in $S$ is updated at least once and for every $i\in S$, all sites in $\mathcal{N}(i)$ are updated at most $\ell\geq 1$ times. We will show that for $\ell= 24\log(d)$, the event $\mathcal{E}_{\ell,\mathcal{N}(n),\tau_{k}}$ holds with probability at least $1/2$ conditioned on $\mathcal{F}_{\tau_{k-1}}$ for any $k\geq 1$.

\begin{corollary}
\label{cor:glauber_unbiased}
    Assume the conditions of \Cref{prop:nbd_unbiased} and let the stopping times $\tau_1< \tau_2<\ldots$ be defined as above. Define $\ell:=24\log d$. Then for every $k\geq 1$, it holds that conditional on $\mathcal{F}_{\tau_{k-1}}$ and the event $\mathcal{E}_{\ell,\mathcal{N}(n),\tau_{k}}$, $X^{\tau_k}_{\mathcal{N}(n)}$ is $\delta$-unbiased for $\delta=c\exp(-6\lambda k \ell)$, where $c>0$ is the constant of \Cref{prop:nbd_unbiased}. Moreover,
    \begin{equation*}
        \Pr(\mathcal{E}_{\ell,\mathcal{N}(n),\tau_{k}}\vert \mathcal{F}_{\tau_{k-1}})\geq 1/2.
    \end{equation*}
\end{corollary}
\begin{proof}
The fact that $X^{\tau_k}_{\mathcal{N}(n)}$ is $\delta$-unbiased with the stated value of $\delta$ is a direct consequence of \Cref{prop:nbd_unbiased} and the Markov property of Glauber dynamics. Indeed, the trajectory on $[\tau_{k-1}+1,\tau_k]$ given $\mathcal{F}_{\tau_{k-1}}$ is equal in law to the trajectory started at $t=0$ with original configuration $X^{\tau_{k-1}}$ until the first update time of node $n$ after $4\log(d)n$ steps have occurred. Therefore, it suffices to prove the stated probability bound for $k=1$ with an arbitrary starting configuration $X^0$. Note though that the event $\mathcal{E}_{\ell,\mathcal{N}(n),\tau_1}$ does not depend on $X^0$ as it depends only on the sequence of updates on $[1,\tau_1]$.

We prove the stated probability bound by checking each of the defining events:
\begin{enumerate}
    \item The probability that there exists a site in $\mathcal{N}(n)$ that does not update before $\tau_1$ is at most $1/d^3$. Indeed, this probability is upper bounded by the probability that there exists a site in $S$ that does not update before $t=4\log(d)n$. For any fixed node $j$, the probability that $j$ does not update in this interval is at most
    \begin{equation*}
        (1-1/n)^{4\log(d)n}\leq \exp(-4\log(d))=\frac{1}{d^4}.
    \end{equation*}
    Since $\vert \mathcal{N}(n)\vert\leq d$ under \Cref{assumption:psi}, the probability that any node in $\mathcal{N}(n)$ fails to update in this interval is thus at most $1/d^3$.

    \item Next, we show that $\tau_1\leq 6\log(d)n$ with probability at least $1-1/d^2$. Indeed, this is a consequence of the previous calculation: after the first $4\log(d)$ steps have occurred, the probability that the next update of node $n$ takes at least $2\log(d)n$ steps is at most $1/d^2$. Therefore, the probability $\tau_1\geq 6\log(d)n$ is at most $1/d^2$ as claimed.

    \item Finally, let $V=\mathcal{N}(\mathcal{N}(n))$ be the set of sites that are a neighbor of some node in $\mathcal{N}(n)$. Note that $\vert V\vert\leq d^2$ under \Cref{assumption:psi}. We now argue that given $\tau_1\leq 6\log(d)n$, the probability there exists $v\in V$ that updates at least $24\log(d)$ times is at most $1/d^2$. Indeed, the number of updates of any such site $v\in V$ on $[1,6\log(d)n]$ is stochastically dominated by a $\mathsf{Bin}(6\log(d)n,1/(n-1))$ random variable. It follows that the probability that there are at least $24\log(d)$ updates is at most the probability that $\mathsf{Bin}(6\log(d)n,1/(n-1))\geq 24\log(d)\geq 2\cdot 6\log(d)\frac{n}{n-1}$ for $n\geq 2$, which by the multiplicative Chernoff bound is at most $\exp(-12\log(d)/3)\leq 1/d^4$. Therefore, a union bound again implies this probability is at most $1/d^2$.
\end{enumerate}
Putting it all together, the probability that any of the conditions fail is at most $3/d^2$, which is at most $1/2$ so long as $d\geq 3$.
\end{proof}

\subsection{Algorithmic Guarantees}
\label{sec:lr_guarantees}

Using the machinery of the previous section in tandem with important deterministic inequalities of Klivans and Meka, we can now show that logistic regression efficiently recovers the parameters of a bounded-degree MRF given the Markov blanket. The key takeaway is that given the neighborhood of a node, the complexity of recovering the parameters only scales with the degree rather than $n$.

We first require the following simple lower bound on the difference in population logistic losses. For notation, let $\mathcal{D}$ be a distribution on $\{-1,1\}^n$ satisfying
\begin{equation}
\label{eq:good_dist_law}
    \Pr(X_n=1\vert X_{-n})=\sigma(2p^*(X_{\mathcal{N}(n)}))
\end{equation}
for some $p^*\in P(d_n,k,\lambda)$. We then define the population logistic loss of any polynomial $q\in P(d_n,k,\lambda)$ by

\begin{equation*}
    \mathcal{L}_{\mathcal{D}}(q) \triangleq \mathbb{E}_{\mathcal{D}}\left[\ell(2\cdot X_n\cdot q(X_{\mathcal{N}(n)}))\right].
\end{equation*}

Then we have the following bound of Wu, Sanghavi, and Dimakis \cite{DBLP:conf/nips/WuSD19}:
\begin{lemma}
\label{lem:wsd_lb}
    Let $\mathcal{D}$ be any distribution on $\{-1,1\}^n$ satisfying \Cref{eq:good_dist_law} for a polynomial $p^*$. Then for any $q\in P(d_n,k,\lambda)$,
    \begin{equation*}
        \mathcal{L}_{\mathcal{D}}(q)-\mathcal{L}_{\mathcal{D}}(p^*)\geq 2\cdot\mathbb{E}_{\mathcal{D}}\left[\left(\sigma(2q(X_{\mathcal{N}(n)}))-\sigma(2p^*(X_{\mathcal{N}(n)}))\right)^2\right]
    \end{equation*}
\end{lemma}

We also require the following result of Klivans and Meka \cite{DBLP:conf/focs/KlivansM17} showing that if low-degree polynomials $p,q$ behave very similarly under the sigmoid function on any $\delta$-unbiased distribution, then their coefficients must be close in $\ell_1$. Their work applies this result to the case where $\mathcal{D}$ is the distribution of a MRF $\mu$; however, in our application to learning from dynamics, it is crucial that the result holds for arbitrary unbiased distributions that may randomly differ across time.
\begin{lemma}[Lemma 6.4 of \cite{DBLP:conf/focs/KlivansM17}]
\label{lem:km_lb}
    There exists a constant $C>0$ such that the following holds. Let $\mathcal{D}$ be a $\delta$-unbiased distribution on $\{-1,1\}^d$. Suppose that $p,q:\{-1,1\}^d\to \mathbb{R}$ are degree $k$ multilinear polynomials and suppose that
    \begin{equation*}
        \mathbb{E}_{X\sim\mathcal{D}}\left[\left(\sigma( p(X))-\sigma( q(X))\right)^2\right]\leq \varepsilon,
    \end{equation*}
    where $\varepsilon\leq \exp(-2\|p\|_1-6)\delta^k$. Then 
    \begin{equation*}
        \|p-q\|_1\leq C\cdot (2k)^k\exp(\|p\|_1)\cdot \binom{d}{k}\cdot \sqrt{\varepsilon/\delta^k}.
    \end{equation*}
\end{lemma}

We may now finally put together all of the above pieces to prove parameter recovery guarantees for logistic regression:

\begin{theorem}
\label{thm:lr_final}
    Consider the logistic regression problem of \Cref{defn:log_reg} given $T$ generated samples for node $n$ at the stopping times of \Cref{cor:glauber_unbiased}. For any $\varepsilon>0$ and $\delta>0$, if $T\geq \tilde{O}\left(\frac{\lambda^2 k\log(d/\delta) d^{O(\lambda k^2)}}{\varepsilon^4}\right)$, then the output of logistic regression yields a polynomial $\widehat{p_n}$ such that $\|\widehat{p_n}-\partial_n\psi\|_{\infty}\leq \varepsilon$ with probability at least $1-\delta$.

    In particular, if the above procedure is done for each $i\in [n]$ with additive accuracy $\varepsilon>0$ with failure probability $\delta/n$, then one may efficiently construct a polynomial $\widehat{p}:\{-1,1\}^n\to \mathbb{R}$ such that $\|\widehat{p}-\psi\|_{\infty}\leq \varepsilon$ with probability at least $1-\delta$.
\end{theorem}

Note that the logistic regression problem \Cref{eq:emp_opt} for a single node $i\in [n]$ is a convex program on $D=O(d^{k-1})$ variables, and thus can be solved for all nodes up to $\varepsilon$ accuracy with $1-\delta$ probability using $n\cdot \mathsf{poly}(T,1/\varepsilon,D)=O(n\cdot\mathsf{poly}(\log(n/\delta),1/\varepsilon))$ via standard methods (again hiding dependence on $d,k,\lambda$). We defer to Wu, Sanghavi, and Dimakis~\cite{DBLP:conf/nips/WuSD19} for more discussion on how to perform fast first-order optimization for this problem via mirror descent.
\begin{proof}[Proof of \Cref{thm:lr_final}]
    Let $\delta=cd^{-O(\lambda k)}$ be the unbiasedness parameter of \Cref{cor:glauber_unbiased}. We will write $\mathcal{D}_{t}$ for the (random) conditional distribution of $(X^{\tau_{t}}_{-n},X^{\tau_{t}}_n)$ given $\mathcal{F}_{\tau_{t-1}}$. If $\varepsilon'\leq \exp(-4\lambda-6)\delta^k$, and the optimizer $\widehat{p}_n$ of the logistic regression problem as in \Cref{eq:emp_opt} satisfies
    \begin{equation}
    \label{eq:close_pop_loss}
        \frac{1}{T}\sum_{t=1}^T\mathcal{L}_{\mathcal{D}_{t}}(\widehat{p}_n)-\frac{1}{T}\sum_{t=1}^T\mathcal{L}_{\mathcal{D}_{t}}(\partial_n\psi)\leq \varepsilon',
    \end{equation}
    then we further have by \Cref{lem:wsd_lb} and \Cref{cor:glauber_unbiased} that
    \begin{align*}
        \varepsilon'&\geq \frac{1}{T}\sum_{t=1}^T\mathcal{L}_{\mathcal{D}_{t}}(\widehat{p}_n)-\frac{1}{T}\sum_{t=1}^T\mathcal{L}_{\mathcal{D}_{t}}(\partial_n\psi)\\
        &\geq \frac{2}{T}\cdot\sum_{t=1}^T\mathbb{E}_{\mathcal{D}_{t}}\left[\left(\sigma(2\widehat{p}_n(X^{\tau_t}_{\mathcal{N}(n)}))-\sigma(2\partial_n\psi(X^{\tau_t}_{\mathcal{N}(n)}))\right)^2\vert \mathcal{F}_{\tau_{t-1}}\right]\\
        &\geq \frac{2}{T}\cdot\sum_{t=1}^T\mathbb{E}_{\mathcal{D}_{t}}\left[\left(\sigma(2\widehat{p}_n(X^{\tau_t}_{\mathcal{N}(n)}))-\sigma(2\partial_n\psi(X^{\tau_t}_{\mathcal{N}(n)}))\right)^2\vert \mathcal{F}_{\tau_{t-1}},\mathcal{E}_{24\log(d),\mathcal{N}(n),\tau_{t}}\right]\cdot \Pr(\mathcal{E}_{24\log(d),\mathcal{N}(n),\tau_{t}}\vert \mathcal{F}_{\tau_{t-1}})\\
        &\geq \frac{1}{T}\cdot\sum_{t=1}^T\mathbb{E}_{\mathcal{D}_{t}}\left[\left(\sigma(2\widehat{p}_n(X^{\tau_t}_{\mathcal{N}(n)}))-\sigma(2\partial_n\psi(X^{\tau_t}_{\mathcal{N}(n)}))\right)^2\vert \mathcal{F}_{\tau_{t-1}},\mathcal{E}_{24\log(d),\mathcal{N}(n),\tau_{t}}\right].
    \end{align*}
    Moreover, conditioned on this event for each $t\geq 1$, \Cref{cor:glauber_unbiased} implies that $X_{\mathcal{N}(n)}^{\tau_{t}}$ is conditionally $\delta$-unbiased. Therefore, \Cref{lem:km_lb} implies that 
    \begin{equation*}
        \|p^*_n-\partial_n\psi\|_1\leq C\cdot (2k)^k \exp(2\lambda){\binom{d}{k}}\sqrt{\varepsilon'/\delta^k}.
    \end{equation*}

    Assuming $\lambda=\Omega(1)$ and rearranging, it follows that $\|\widehat{p}_n-\partial_n\psi\|_1\leq \varepsilon$ so long as 
    \begin{equation*}
        \varepsilon' \leq c\varepsilon^2 d^{-C\lambda k^2}
    \end{equation*}
    for any $\varepsilon>0$ with absolute constants $c,C>0$. The final claim then follows by simply defining the polynomial $\widehat{p}$ via $\widehat{p}(I)= \widehat{p}_i(I)$ for $i=\arg\min I$.

    Therefore, it suffices to show that if $T$ is as stated in the theorem statement, then \Cref{eq:close_pop_loss} holds with probability at least $1-\delta$. For any $T\geq 1$, if $u$ is set to $C\lambda \sqrt{T\log(2/\delta)}$ for an appropriate constant $C$, then with probability at least $1-\delta$, we have the following sequence of inequalities:
    \begin{align*}
        \sum_{t=1}^T\mathcal{L}_{\mathcal{D}_{t}}(\widehat{p}_n)&\leq \sum_{t=1}^T\ell(2\cdot Y^{\ell}\cdot \widehat{p}_n(X^{\tau_t}_{\mathcal{N}(n)})) + C\lambda \log^{3/2}(T)\sqrt{Tk\log(d/\delta)}\\
        &\leq \sum_{t=1}^T\ell(2\cdot Y^{\ell}\cdot \partial_n\psi(X^{\tau_t}_{\mathcal{N}(n)})) + C\lambda \log^{3/2}(T)\sqrt{Tk\log(d/\delta)}\\
        &\leq \sum_{t=1}^T\mathcal{L}_{\mathcal{D}_{t}}(\partial_n\psi)+C'\lambda \log^{3/2}(T)\sqrt{Tk\log(d/\delta)}
    \end{align*}
    for a slightly different constant $C'>0$. Here, we use the guarantee of \Cref{thm:uniform_convergence} for the first inequality with the choice of $u$, then the fact that $\widehat{p}_n$ is defined to be the optimizer of \Cref{eq:emp_opt}, and finally the two-sided bound of \Cref{lem:fixed_good} for the fixed polynomial $\partial_n\psi$. Dividing by $T$, it follows that so long as
    \begin{equation*}
T\geq \tilde{O}\left(\frac{\lambda^2 k\log(d/\delta) d^{O(\lambda k^2)}}{\varepsilon^4}\right),
    \end{equation*}
    then \Cref{eq:close_pop_loss} will indeed hold, completing the proof.
\end{proof}

\section{Conclusion}

In this work, we have shown that leveraging correlations in a natural observation model can provably overcome notorious computational barriers for the classical problem of learning MRFs in the i.i.d. setting. As described above, the i.i.d. assumption is a mathematically convenient condition that often enables provable algorithmic guarantees. However, not only can considering models with correlations remedy the practical deficiencies of this assumption, but also these correlations can provide algorithmic footholds for better rigorous gurantees. Finding more important statistical settings, new and old, where similar phenomena occur is an exciting direction for future research.

We suspect that much of our analysis can be sharpened---while we ensure that our algorithm is efficient for higher $k>2$, our current bounds are rather pessimistic. Determining the correct parameter dependencies is an immediate direction for future work. A more extensive experimental evaluation of the practical difficulty of learning MRFs from dynamics compared to the i.i.d. setting would also be interesting.

A natural question is to extend the provable guarantees of learning from processes like Glauber dynamics to less restrictive settings. It is easily seen that our analysis does not really require the assumption that nodes update at precisely the same Poisson rate; most natural update rates will suffice. However, a critical assumption (inherited in all previous work on provable learning from dynamics for the Ising model) is that one observes whether a variable attempts to update even when the resampled value itself does not change. In practice, one expects that this would only possible for a possibly small subset of nodes. Developing algorithms that can succeed for weaker or alternative observation assumptions of this type is an important direction for future work.


\bibliographystyle{alpha}
\bibliography{bibliography}

\appendix

\section{Anti-concentration of Glauber Dynamics}
\label{sec:app1}

We provide the deferred proof of \Cref{lem:anti-concentration_main}, as well as some discussion of the tightness of this result. We require the following two simple facts:
\begin{fact}
\label{fact:sv_lb}
    If a distribution $\mu$ on $\{-1,1\}^k$ is a $\delta$-Santha-Vazirani source, then for any $\bm{x}\in \{-1,1\}^k$, $\Pr_{X\sim \mu}(X=\bm{x})\geq \delta^k$.
\end{fact}

\begin{fact}
\label{fact:pol_geq_coefficient}
    Let $f:\{-1,1\}^n$ be a multilinear polynomial. Then for any $S\subseteq [n]$, there exists $\bm{x}\in \{-1,1\}^n$ such that $\vert f(\bm{x})\vert\geq \vert \widehat{f}(S)\vert$.
\end{fact}
\begin{lemma}[\Cref{lem:anti-concentration_main}, restated]
    Let $f:\{-1,1\}^n\to \mathbb{R}$ be supported on $[d]$ with degree at most $k$ and let $S$ be a maximal monomial of $f$ and $T> 0$. Suppose that $\mu=\mu_{\psi}$ is an MRF such that the conditional distribution of any site with any outside configuration is uniformly lower bounded by $\delta$. Let $\mu^T$ be the law of $X^T$ after running continuous-time Glauber dynamics on $I=[0,T]$ with some arbitrary initial configuration $X_0$. Further, let $\mathcal{E}_S$ denote the event that every $i\in S$ is updated by the dynamics (i.e. $\Pi_i\cap I\neq \emptyset$ for all $i\in S$). Then for any $T'\geq 0$,
    \begin{equation*}
        \Pr_{X\sim \mu^{T}}\left(\vert f(X^T)\vert\geq \vert \widehat{f}(S)\vert\bigg\vert \mathcal{E}_S,\{\Pi_j(T')\}_{j\in [n]\setminus [d]}\right)\geq \left(\frac{\delta}{d}\right)^k.
    \end{equation*}
\end{lemma}
\begin{proof}
    Let $U\subseteq [d]$ denote the random set of variables that $f$ depends on that are updated over the trajectory. Note that $\vert U\vert\leq d$ and $S\subseteq U$ conditioned on $\mathcal{E}_S$. Let $A_{S}$ denote the event that the set of $k$ variables in $U$ with last \emph{final} update time is $S$. Since the update times of elements in $[d]$ is independent of the update times of elements outside of $[d]$,
    \begin{align*}
        \Pr\left(A_S\vert \mathcal{E}_S,\{\Pi_j(T')\}_{j\in [n]\setminus [d]}\right)&=\Pr\left(A_S\vert \mathcal{E}_S\right)\\
        &=\frac{1}{\Pr(\mathcal{E}_S)}\sum_{V\supseteq S}\Pr(A_S\land U=V)\\
        &=\frac{1}{\Pr(\mathcal{E}_I)}\sum_{V\supseteq S}\Pr(A_S\vert U=V)\Pr(U=V)\\
        &=\frac{1}{\Pr(\mathcal{E}_S)}\sum_{V\supseteq S}\frac{\Pr(U=V)}{\binom{\vert V\vert}{k}}\\
        &\geq \frac{1}{\binom{d}{k}}.
    \end{align*}
    Here, we use the fact that given $U=V\supseteq S$, all subsets of variables with the last $k$ final update times are equally likely.

    Now, given $A_S$ occurs, let $\tau$ denote the last time a variable in $[d]\setminus S$ is chosen for updating in $I$; note that $\tau$ depends only on the \emph{sequence} of updates in $I$ and not the actual value of the updates. Let $X_{S^c}^{\tau}\in \{-1,1\}^{[d]\setminus S}$ be the random setting of the variables in $[d]\setminus S$ at $\tau$. By definition of $A_S$, all updates after $\tau$ from variables in $[d]$ come from $S$ and each such variable is updated after $\tau$. 
    
    Note that conditional on $A_S$, the law of $X^{T}$ is \emph{not} independent of the updates of the variables outside of $[d]$ until time $T'$, even when restricting to the relevant coordinates, since the updates to those variables can affect the distribution of $X^T$. However, conditional on $A_S$, any configuration $X^{\tau}$ and the update times outside $[d]$, $X^T_{S}$ is a $\delta$-Santha-Vazirani source since the last Glauber update to each site in $S$ has probabilities lower bounded by $\delta$ by assumption on $\mu$. Moreover, for any fixed $\bm{x}_{S^c}$, the restriction of $f$ given via $\bm{x}_S\to f(\bm{x}_{S},\bm{x}_{S^c})$ is a multilinear polynomial with the monomial $\widehat{f}(S)$ by maximality of $S$ in the Fourier expansion of $f$. The claim then follows from \Cref{fact:pol_geq_coefficient} and \Cref{fact:sv_lb}.
\end{proof}

\begin{remark}
    The above argument is essentially tight if one only assumes that (i) the order of updated variables is a uniform permutation, (ii) the variables in $I$ are $\delta$-unpredictable given the previously updated variables, and (iii) the variables outside $I$ can be set arbitrarily depending on the previous updates. To see this, suppose $k$ divides $d$ and consider the function $f(x):=\mathsf{AND}(x_1,\ldots,x_k)-\sum_{j=1}^{d/k-1}\sum_{i=1}^k \mathsf{AND}(x_{[k]\setminus \{i\}},x_{jk+i})$. In this construction, each variable $x_i$ for $i\leq k$ is associated with $d/k-1$ variables $x_{jk+i}$ for $j=1,\ldots,d/k-1$. We say these latter variables and the corresponding summands are \emph{auxiliary}.
    
    Let $I=[k]$ and suppose the variables in $I$ are set to $+1$ with probability $\delta$ independently of all previous updates. We consider the following strategy for setting the variables outside $I$ to make the polynomial small. Let $\mathcal{E}$ be the event that there exists a variable of the form $x_{jk+i}$ with $j\geq 1$ that is chosen to be updated at some point after the variable $x_{i}$ has been updated. Let $\tau$ denote the \emph{first} time this event occurs on $\mathcal{E}$ and let $\tau=d+1$ otherwise. Before and after $\tau$, set each auxiliary variable chosen for updating to $-1$, and set $x_{jk+i}=x_i$ at time $\tau$. It is straightforward to see that if $\mathcal{E}$ occurs, then the remaining function is zero. Indeed, since all of the auxiliary summands get set to $0$ except for the one associated with $x_{jk+i}$. Since the value of the auxiliary variable $x_{jk+i}$ is set to $x_i$ by construction, this surviving auxiliary term cancels the positive $\mathsf{AND}$ term on $I$ regardless of the value of $x_1,\ldots,x_k$. Thus, if $\mathcal{E}$ occurs, the resulting function is surely zero.

    On the other hand, the probability that $\mathcal{E}$ \emph{fails to occur} is precisely $(k/d)^k$. This holds because for $\mathcal{E}$ to fail, each $x_i$ for $i\leq k$ must be the last variable of its associated group in the random ordering. Moreover, if $\mathcal{E}$ fails, the polynomial will be nonzero under this strategy if and only if each variable in $I$ updates to $+1$ since all the auxiliary summands all become $0$. Thus, the overall probability that $f(\bm{x})\neq 0$ is $\delta^k\cdot(k/d)^k\sim \delta^k/{\binom{d}{k}}$ up to $\mathsf{poly}(k)$ factors if $k^2\ll d$.
\end{remark}

\begin{remark}
    In general, some $\tilde{O}(1/d)$ dependence is necessary when $k=1$ if one only uses the facts that (i) the order of updated variables is a uniform permutation, and (ii) all variables are $\delta$-unpredictable conditioned on the previously updated variables. To see this, consider the polynomial $f(x)=C\log(d)x_1+x_2+\ldots+x_d$ for a large constant and take $\delta=1/e$ for convenience. Suppose that each of the $x_i$ is sampled to be biased towards the opposite sign of the previously sampled partial sum. On the prefix before $x_1$ is sampled in the random ordering, the absolute value of the partial sum forms a  biased random walk on $\mathbb{N}$ towards the origin. It is elementary to see that such a process of length at most $d$ is at most $c\log(d)$ with probability at least $1-1/d^{100}$ for some smaller constant $c<C$ so long as $C$ was chosen large enough. Once $x_1$ is sampled, the absolute partial sum deterministically remains at most $2C\log(d)$ so long as $C>0$ was chosen large enough. If the suffix after $x_1$ consists of at least $C'\log(d)$ remaining variables to be sampled for a large enough constant $C'>0$ (say $100C$), an elementary concentration argument implies that the remainder of the random walk will return to zero with probability at least $1-1/d^{100}$. The same argument implies that the remainder of the random walk on at most $d$ variables will remain $c\log(d)$ with probability at least $1-1/d^{100}$. 
    
    This procedure thus only fails if either there are fewer than $C'\log(d)$ variables after $x_1$ in the uniform ordering, or with probability at most $3/d^{100}$ if this does not happen. Thus, the overall probability that $\vert f(\bm{x})\vert\geq C\log(d)$ is at most $O(\log(d)/d)$ since the position of $x_1$ is uniform in $[d]$.
\end{remark}

\section{Proof of Dependency Lower Bound}
\label{sec:app_lb}

In this section, we provide the deferred proof of \Cref{thm:local_lb}. 

\begin{theorem}[\Cref{thm:local_lb}, restated]
    For any fixed $\alpha>0$, there is $\lambda=\lambda(\alpha)>0$ sufficiently large such that there exists a pair of $(3,3,\alpha,\lambda)$-MRFs $\mu_1$ and $\mu_2$ with Hamiltonians $\psi_1$ and $\psi_2$ such that $1\sim 2$ in $\mu_{1}$, $1\not\sim 2$ in $\mu_{2}$, but the distribution of the output of any $(M,1)$-local dependency algorithm evaluated on the pair $(1,2)$ is the same in $\mu_{1}$ and $\mu_{2}$.
\end{theorem}
\begin{proof}[Proof of \Cref{thm:local_lb}]
    It suffices to exhibit a pair of MRFs $\mu_1$ and $\mu_2$ satisfying the theorem statement such that the law $\mathcal{D}$ of $Y=(Y_+,Y_-)\in \{-1,1\}^2$ as in \Cref{defn:local_test} is identical in both models. Since the samples across $m=1,\ldots,M$ are independent and drawn from $\mathcal{D}$, this will establish the claim even given the exact distribution $\mathcal{D}$ on $\{-1,1\}^2$ (which corresponds to the infinite sample limit $M=\infty$).

    Consider the following choices of $\psi_1$ and $\psi_2$ for choices of $\alpha,\beta>0$ to be chosen later:
    \begin{gather*}
        \psi_1(\bm{x})=\beta x_1x_3+\alpha x_1x_2x_4+\beta x_2x_5\\
        \psi_2(\bm{x})=\alpha x_1x_3+\alpha x_2x_3 +\alpha x_4x_5.
    \end{gather*}
    Let $\mu_1,\mu_2$ denote the respective MRFs. We first establish some simple facts about these Gibbs distributions:
    \begin{claim}
    \label{claim:unif_1}
        The distribution of $(X_3,X_4)$ in $\mu_1$ is uniform over $\{-1,1\}^2$.
    \end{claim}
    \begin{proof}
        For any choice of $(x_3,x_4)\in \{-1,1\}^2$, consider the restricted partition function
        \begin{align*}
            Z_{x_3,x_4} &= \sum_{x_1,x_2,x_5\in \{-1,1\}} \exp(\beta x_1x_3+\alpha x_1x_2x_4+\beta x_2x_5)\\
            &\propto \mathbb{E}_{x_1,x_2,x_5}[\exp(\beta x_1x_3+\alpha x_1x_2x_4+\beta x_2x_5)],
        \end{align*}
        where the expectation is taken over uniformly random signs. Notice though that for any fixed $(x_3,x_4)\in \{-1,1\}^2$, the distribution of $(x_1x_3,x_1x_2x_4,x_2x_5)$ is uniform over $\{-1,1\}^3$ when $(x_1,x_2,x_5)$ are independent uniform signs. Thus, the restricted partition function does not depend on the value of $(x_3,x_4)\in \{-1,1\}^2$ and so the law of $(X_3,X_4)$ is uniform under $\mu_1$. 
    \end{proof}

    \begin{claim}
    \label{claim:unif_2}
        The distribution of $X_3$ in $\mu_2$ is uniform over $\{-1,1\}$.
    \end{claim}
    \begin{proof}
        Simply observe that the potential is invariant under the sign reversal $\bm{x}\mapsto -\bm{x}$.
    \end{proof}

    We now consider the law $\mathcal{D}$ of $(Y_+,Y_+)$ as in \Cref{defn:local_test} with the pair of indices 1 and 2. By construction, the role of these sites is symmetric in both models, so we may consider a sample of $X_1$ where we vary the value of $X_2$ given the randomness of the remaining variables in the Gibbs distribution.
    
    In $\mu_1$, because the conditional distribution of $X_1$ only depends on $X_2,X_3,X_4$, for any fixed $x_3,x_4$, we have the following distribution over $\{-1,1\}^2$:

    \begin{table}[htbp]
    \centering
    \caption{Distribution of $(Y_+,Y_-)$ in $\mu_1$ given $x_3,x_4$.}
    \label{tab:contingency_table}
    \begin{tabular}{c|cc}
        \hline
        & $Y_- = +1$ & $Y_- = -1$ \\
        \hline
        $Y_+ = +1$ & $\sigma(2(\beta x_3+\alpha x_4))\sigma(2(\beta x_3-\alpha x_4))$ & $\sigma(2(\beta x_3+\alpha x_4))\sigma(2(-\beta x_3+\alpha x_4))$ \\
        $Y_+ = -1$ & $\sigma(2(-\beta x_3-\alpha x_4))\sigma(2(\beta x_3-\alpha x_4))$ & $\sigma(2(-\beta x_3-\alpha x_4))\sigma(2(-\beta x_3+\alpha x_4))$ \\
        \hline
    \end{tabular}
\end{table}

Similarly, the following distribution for $(Y_+,Y_-)$ on $\{-1,1\}^2$ under $\mu_2$ given $x_3$ is given in \Cref{tab:contingency_table_2}, noting that the distribution of $X_1$ depends only on $x_3$.

\begin{table}[h]
    \centering
    \caption{Distribution of $(Y_+,Y_-)$ in $\mu_2$ given $x_3$.}
    \label{tab:contingency_table_2}
    \begin{tabular}{c|cc}
        \hline
        & $Y_- = +1$ & $Y_- = -1$ \\
        \hline
        $Y_+ = +1$ & $\sigma(2\alpha x_3)^2$ & $\sigma(2\alpha x_3)\sigma(-2\alpha x_3)$ \\
        $Y_+ = -1$ & $\sigma(2\alpha x_3)\sigma(-2\alpha x_3)$ & $\sigma(-2\alpha x_3)^2$ \\
        \hline
    \end{tabular}
\end{table}

Since by \Cref{claim:unif_1} the law of $(X_3,X_4)$ is uniform under $\mu_1$, the unconditional distribution of $(Y_+,Y_-)$ is such that the diagonal elements are equal, as are the off-diagonals. The same holds true for $\mu_2$ by \Cref{claim:unif_2}, and thus to show that $(Y_+,Y_-)$ has the same law in both models, it suffices to equate the probability that $(Y_+,Y_-)=(1,1)$ in both models by choosing parameters appropriately. 

We claim there exists $\beta=\beta(\alpha)> 0$ such that this is the case. Indeed, consider the function $g(\beta)$ given by $\beta\mapsto \mathbb{E}_{x_3,x_4}[\sigma(2(\beta x_3+\alpha x_4))\sigma(2(\beta x_3-\alpha x_4))]$, which is this probability for $\mu_1$. For $\beta=0$, we have 

\begin{align*}
g(0)&=\sigma(2\alpha)\sigma(-2\alpha)\\
&\leq \frac{1}{2}(\sigma^2(2\alpha)+\sigma^2(-2\alpha)^2)\\
&=\mathbb{E}_{x_3}[\sigma(2\alpha x_3)^2]\\
&<1/2-\eta(\alpha),
\end{align*} 
for some $\eta(\alpha)>0$ if $\alpha>0$ where we use Young's inequality, and where the first inequality is strict if $\alpha>0$ and for some $\eta(\alpha)>0$. Note that the third line is precisely the unconditional probability that $(Y_+,Y_-)=(1,1)$ in $\mu_2$.
By continuity, the existence of the desired $\beta=\beta(\alpha)$ will thus follow from showing that $g(\beta)\geq 1/2-\eta(\alpha)$ for sufficiently large $\beta>0$. Indeed, it holds that for $\beta>2\alpha$,
\begin{equation*}
    g(\beta)\geq \frac{1}{2}\left(\sigma(2(\beta+\alpha))^2+\sigma(-2(\beta+\alpha))^2\right)-c\alpha\exp(-\beta/2),
\end{equation*}
for some constant $c>0$ since $\sigma(\cdot)$ is nonnegative and monotone and because the function $h(z)=\sigma(2(\beta-\alpha+z))$ for $z\geq 0$ is $c'\exp(-\beta/2)$-Lipschitz if $\beta>2\alpha$ for a constant $c'>0$ by differentiating. Taking $\beta$ sufficiently large certifies the claim since the first term tends to $1/2$ while the latter vanishes. 
\end{proof}

\section{Experimental Results}
\label{sec:experiments}

\begin{figure}
    \centering
    \begin{minipage}{0.5\textwidth}
        \centering
\includegraphics[scale=.5]{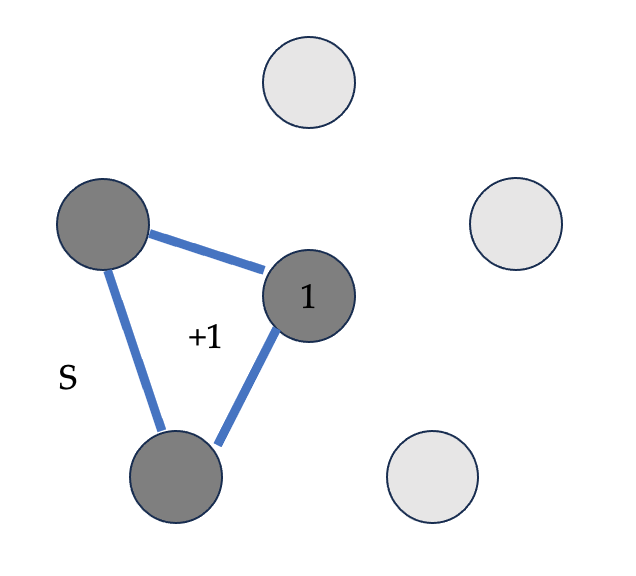} 
    \end{minipage}\hfill
    \begin{minipage}{0.5\textwidth}
\includegraphics[scale = .5]{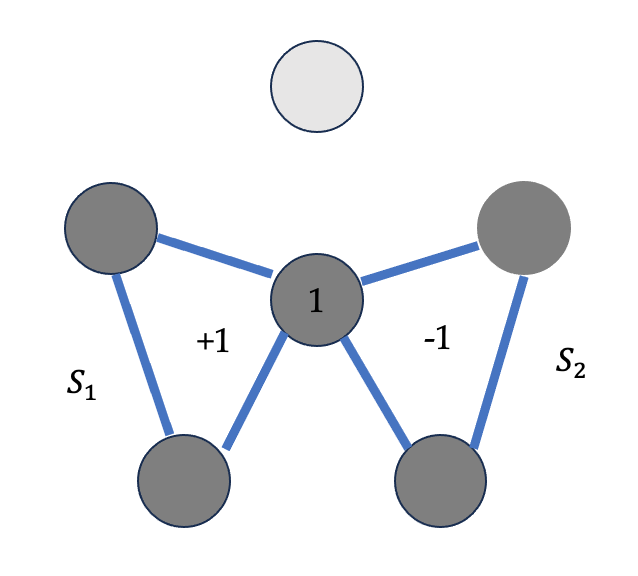} 
    \end{minipage}
    \caption{\label{figure:examples}We consider the performance of \Cref{alg:markov_blanket} compared to Sparsitron on random SPN instances (left), as well as demonstrate that \Cref{alg:markov_blanket} succeeds on two parity instances (right).}
\end{figure}
While our main contribution is theoretical, we provide preliminary experimental results demonstrating the feasibility of \Cref{alg:markov_blanket} compared to the i.i.d. setting. Providing a more extensive and larger-scale empirical evaluation of the learnability of MRFs from dynamical vs. i.i.d. samples is an important direction for future research.

\begin{figure}
    \centering
    \begin{minipage}{0.45\textwidth} \hspace*{-1.5cm}
\includegraphics[width=10.5cm, height = 6.9cm]{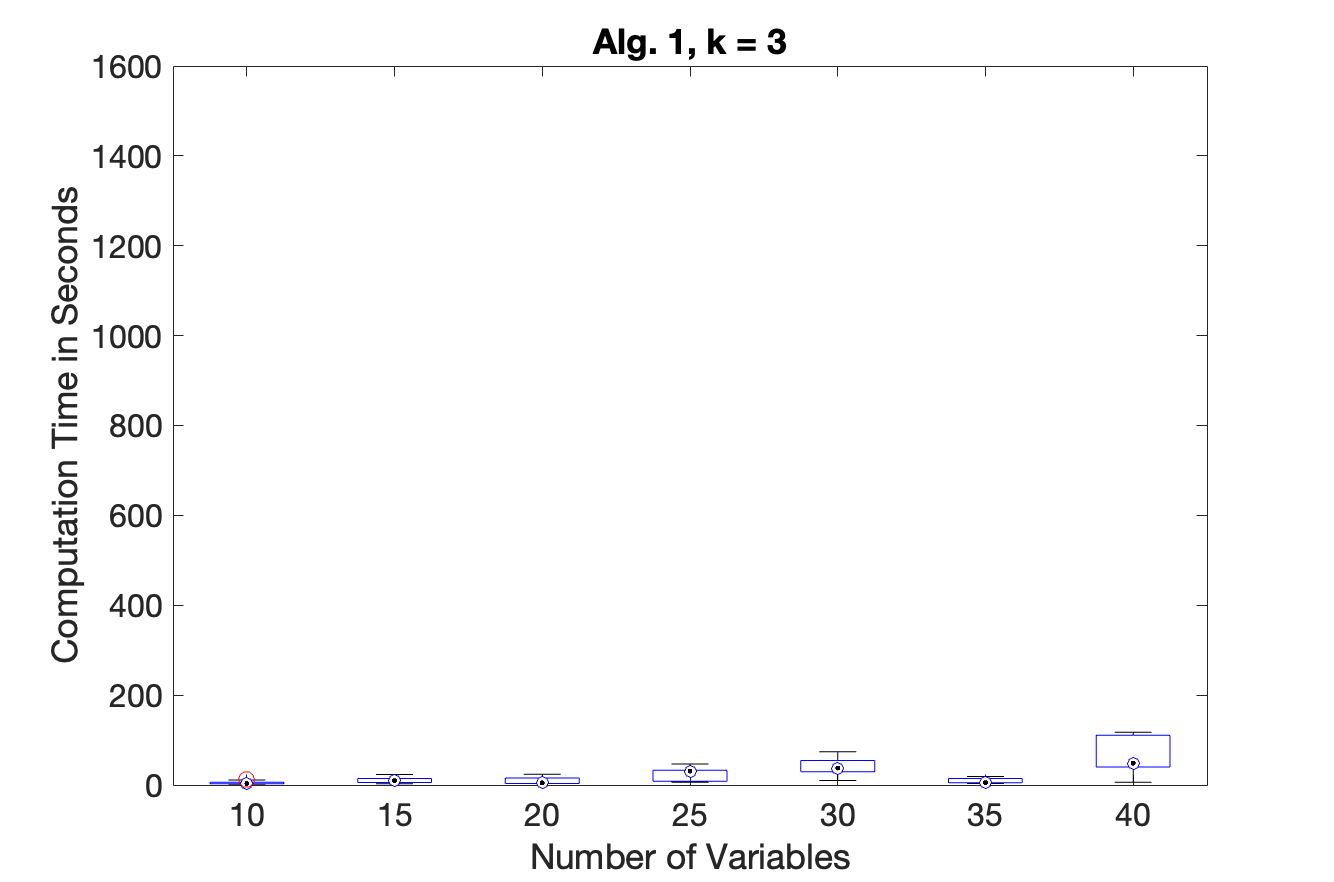} 
    \end{minipage}\hfill
    \begin{minipage}{0.49\textwidth}
\includegraphics[width = 10.5cm, height = 6.9cm]{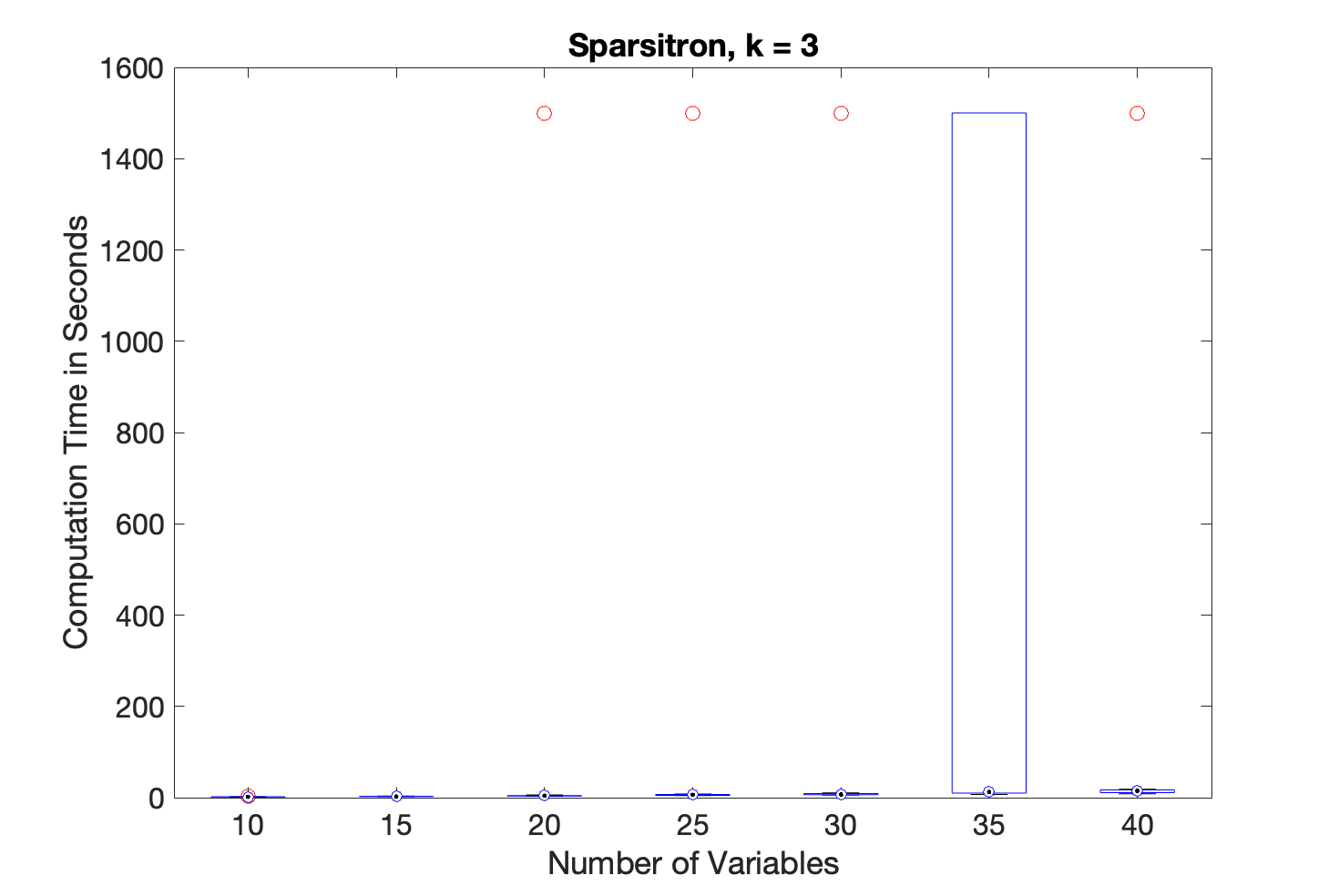} 
    \centering
    \end{minipage}\\
    \begin{minipage}{0.49\textwidth}\hspace*{-1.5cm}
\includegraphics[width = 10.5cm, height = 6.9cm]{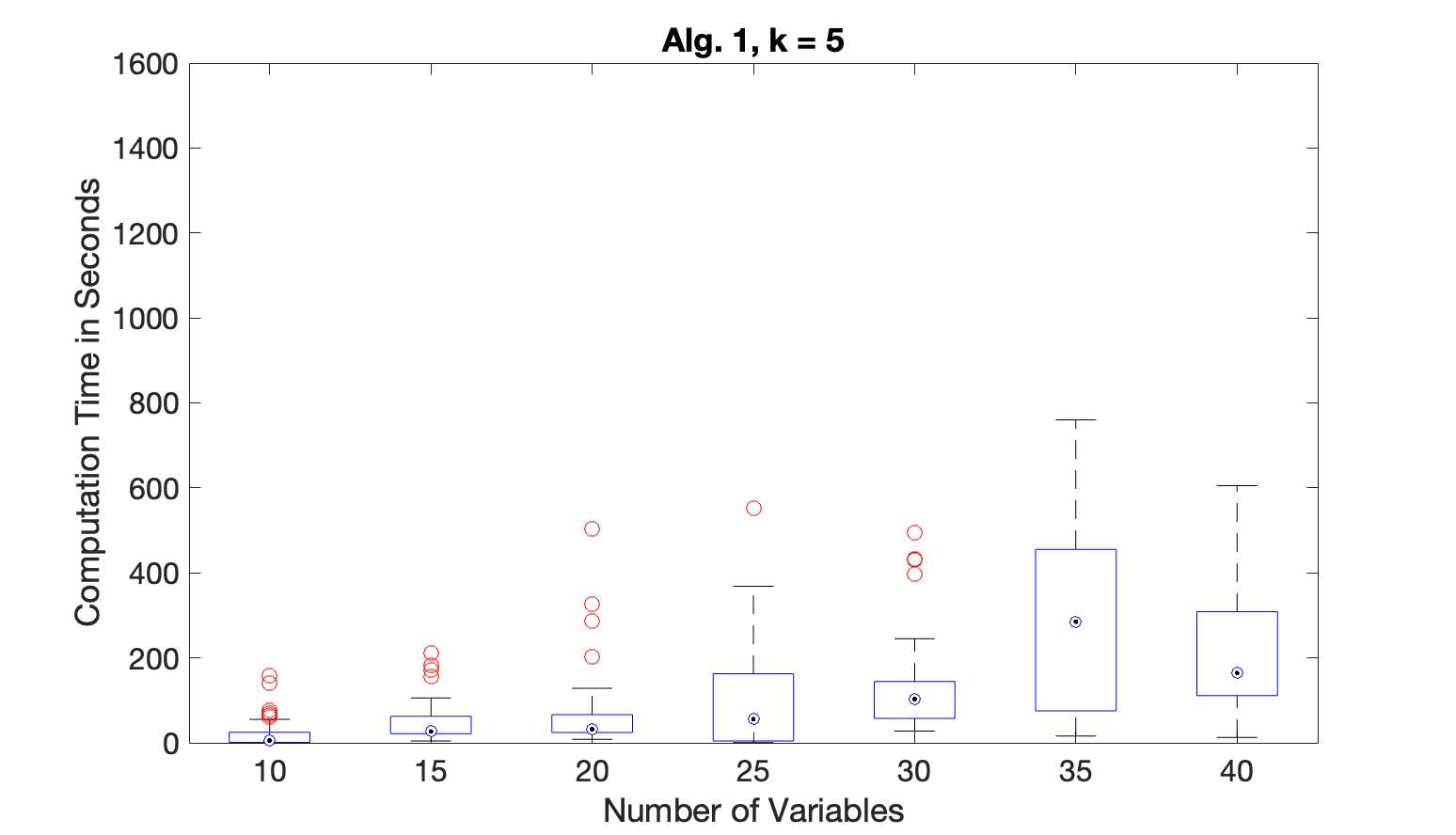} 
    \end{minipage}\hfill
    \begin{minipage}{0.49\textwidth}
\includegraphics[width = 10.5cm]{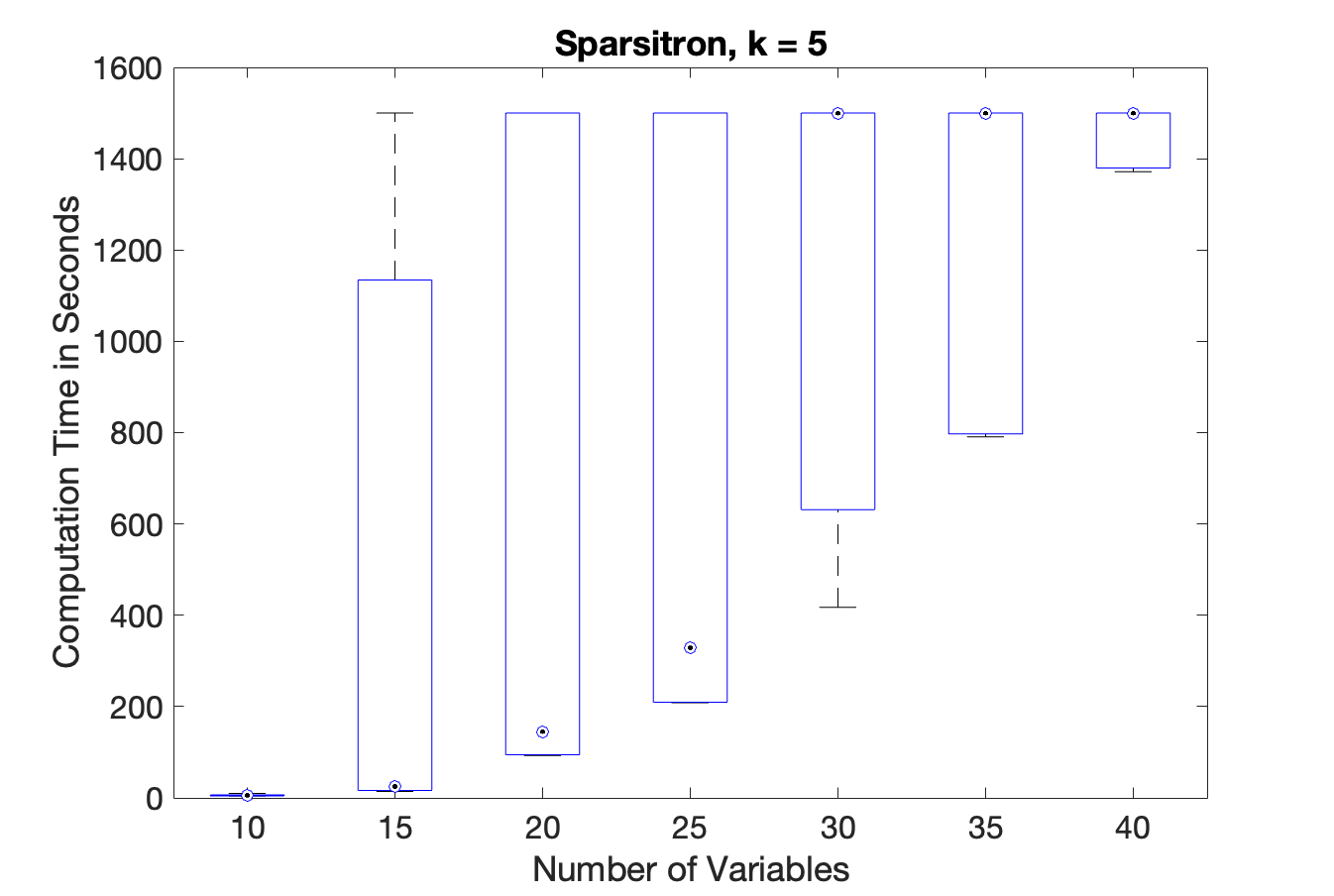} 
    \centering
    \end{minipage}\\
        \begin{minipage}{0.49\textwidth}\hspace*{-1.5cm}
\includegraphics[width=10.5cm, height = 6.9cm]{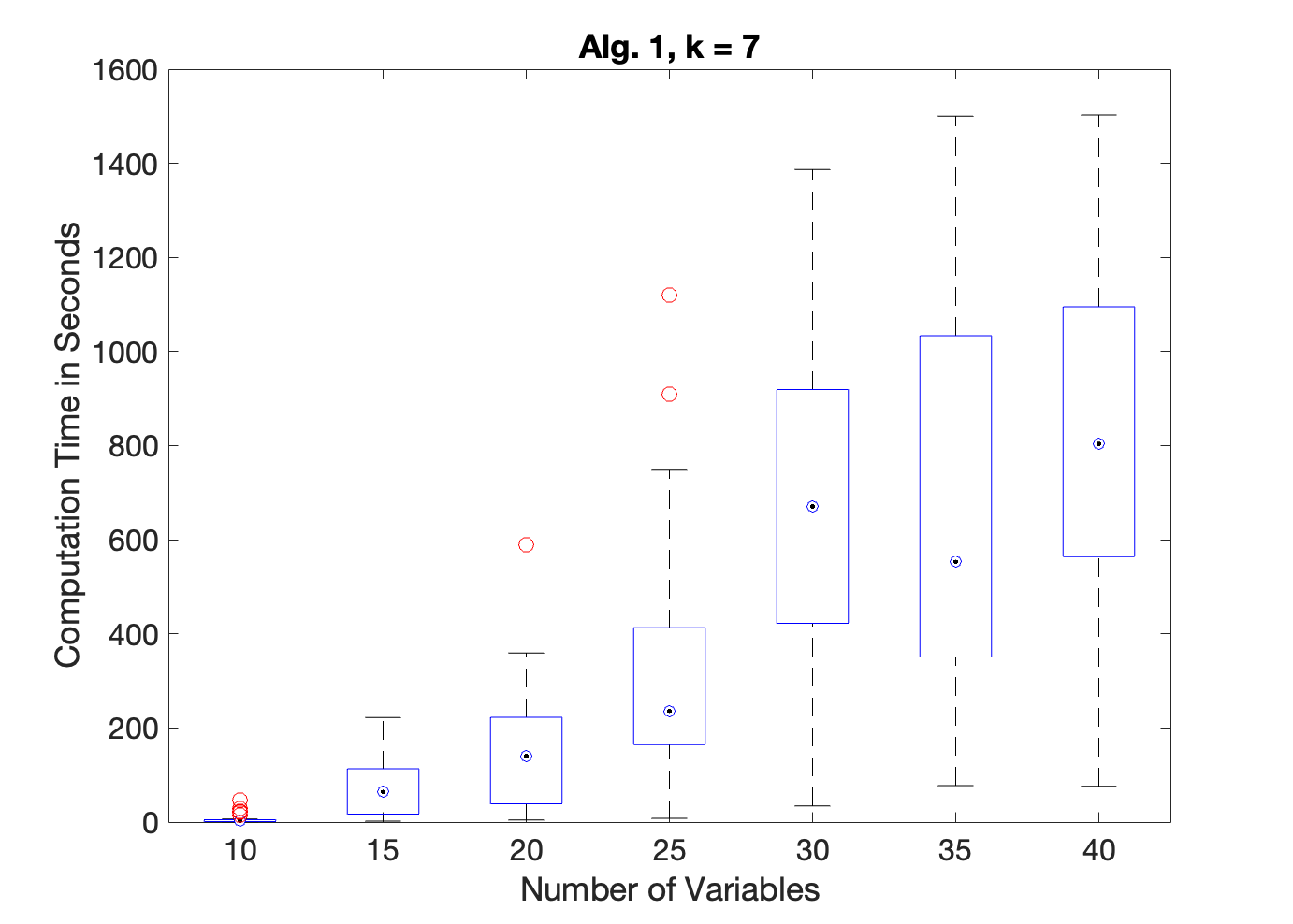} 
    \end{minipage}\hfill
    \begin{minipage}{0.49\textwidth}
\includegraphics[width = 10.5cm, height = 6.9cm]{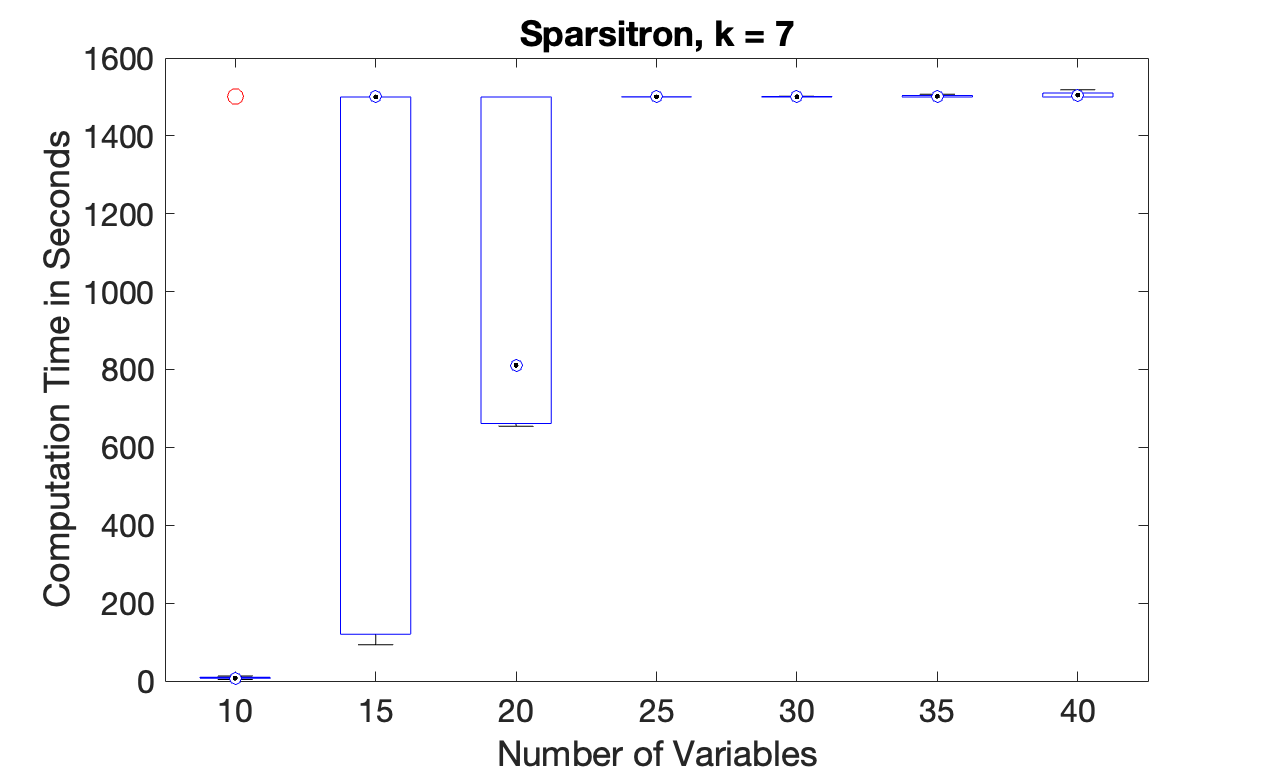} 
    \centering
    \end{minipage}
    \caption{\label{figure:experiments}We consider the performance of \Cref{alg:markov_blanket} (left) compared to Sparsitron (right) on random SPN instances, where each run was stopped after 1500 seconds if it did not stably achieve approximate neighborhood recovery. The interquartile boxes plot the middle 50\% of runtimes, the dotted circles denote the median times, and the red circles denote outlier runtimes. While Sparsitron performs slightly quicker for $k=2$ and small values of $n$, the runtime quickly blows up due to the $\Theta(n^k)$ time and space complexity.}
\end{figure}

\paragraph{SPN Instances} We compare our structure learning algorithm in \Cref{alg:markov_blanket} to the existing Sparsitron algorithm of Klivans and Meka \cite{DBLP:conf/focs/KlivansM17} for the i.i.d. case for the important special case of SPN instances (see left side of Figure~\ref{figure:examples}). Concretely, we perform the following experiment. For each $n\in \{10,15,20,25,30,35,40\}$ and $k\in \{3, 5, 7\}$, we plant a random parity of size $k$ containing node $1$. In particular, we consider the Hamiltonian $\psi(\bm{x})=x_1\prod_{i\in S} x_i$, where $S$ is a uniformly random subset in $\{2,\ldots,n\}$ of size $k-1$. 

For space efficiency, we run \Cref{alg:markov_blanket} on iterative blocks of Glauber dynamics of length $T=10000$ for several iterations. We consider our algorithm successful if for at least $10$ consecutive blocks, it holds that top $k$ neighbors found by \Cref{alg:markov_blanket} has overlap at least $\lceil3k/4\rceil$ with $S$. While \Cref{alg:markov_blanket} requires particular settings for the window size $L$ and the threshold $\kappa$, these bounds are rather pessimistic; therefore, in our tests, we set the (discrete-time) window size to be $L=3\cdot \max\{2,\lceil n/k\rceil\}$ and use the above heuristic rather than an explicit threshold.

For Sparsitron, we provide i.i.d. SPN samples of the same form; note that for general MRFs, generating i.i.d. samples from $\mu$ may be experimentally challenging, but planted SPN instances are easy to exactly sample from. The Sparsitron algorithm maintains a distribution over all $\Theta(n^k)$ monomials of size at most $k$ that approximates $\psi$ from samples via the Multiplicative Weights algorithm and outputs the best previous distribution over monomials evaluated on a separate test set of size $1000$. We similarly run their algorithm in blocks of $1000$ samples and check if any of the 3 highest weight monomials in the current optimizer has symmetric difference with the true parity as above. If this holds for 5 blocks in a row, we consider the algorithm to succeed.

For each algorithm and choice of $(n,k)$, we repeat the experiments for at least $30$ trials, which were conducted on a personal laptop in Matlab. We plot the time it takes for both algorithms to succeed in the sense described above in Figure~\ref{figure:experiments}, capping the runtimes to $1500$ seconds. We only record the time it takes for the actual algorithm to run given the data without including the time required to generate samples. We also report the fraction of trials that each algorithm succeeded in approximate neighborhood recovery in Table~\ref{table:successes}. We find that \Cref{alg:markov_blanket} succeeds in approximate neighborhood recovery in nearly all instances well within the allotted time. We expect that with more tailored choices of window size, our algorithm will succeed on significantly larger instances. Meanwhile, Sparsitron often succeeds quite quickly for $k=2$, but is less accurate as $n$ increases. It then drastically slows down for $k=4$ even on small $n$ values and then fails for any larger value of $n$ or $k$ in the allotted time. We remark that with 60GBs of memory, Sparsitron cannot run for significantly larger values of $n$ when $k=6$ due to the $\Theta(n^k)$ storage requirement.

\setcounter{table}{0}
\begin{table}[h]
\centering
\ra{1.3}
\begin{tabular}{@{}rrrrcrrrcrrr@{}}\toprule
& \multicolumn{2}{c}{$k = 3$} & \phantom{abc}& \multicolumn{2}{c}{$k = 5$} &
\phantom{abc} & \multicolumn{2}{c}{$k = 7$}\\
\cmidrule{2-3} \cmidrule{5-6} \cmidrule{8-9}
& Alg. 1 & Sparsitron && Alg. 1 & Sparsitron  && Alg. 1 & Sparsitron\\
\midrule
$n = 10$ & $\mathbf{1.00}$ & $\mathbf{1.00}$  && $\mathbf{1.00}$ & $\mathbf{1.00}$  && $\mathbf{1.00}$ & $0.91$ \\
$n = 15$ & $\mathbf{1.00}$ & $\mathbf{1.00}$ && $\mathbf{1.00}$ & $0.74$ && $\mathbf{1.00}$ & $0.41$\\
$n = 20$ & $\mathbf{1.00}$ & $0.98$&& $\mathbf{1.00}$ & $0.58$&& $\mathbf{1.00}$ & $0.59$\\
$n = 25$ & $\mathbf{1.00}$ & $0.90$&& $\mathbf{1.00}$ & $0.52$&& $\mathbf{1.00}$ & $0.00$\\
$n = 30$ & $\mathbf{1.00}$ & $0.92$&& $\mathbf{1.00}$ & $0.29$&& $\mathbf{1.00}$ & $0.00$\\
$n = 35$ & $\mathbf{1.00}$ & $0.73$&& $\mathbf{1.00}$ & $0.48$&& $\mathbf{0.97}$ & $0.00$\\
$n = 40$ & $\mathbf{1.00}$ & $0.79$&& $\mathbf{1.00}$ & $0.45$&& $\mathbf{0.90}$ & $0.00$\\
\bottomrule
\end{tabular}
\caption{\label{table:successes}Empirical success probabilities of \Cref{alg:markov_blanket} and Sparsitron on SPN instances of order $k$ with $n$ variables within the $1500$ seconds time bound. The success probability of Sparsitron is high for small $n$ and remains reasonable for $k=3$, but quickly degrades for larger $n$ and higher $k$. The success probability of \Cref{alg:markov_blanket} remains high across all tested parameters, only failing a small fraction of time for $n\geq 35$ and $k=7$.}
\end{table}

\begin{figure}
    \centering
    \begin{minipage}{0.49\textwidth}
        \centering
\includegraphics[scale=.36]{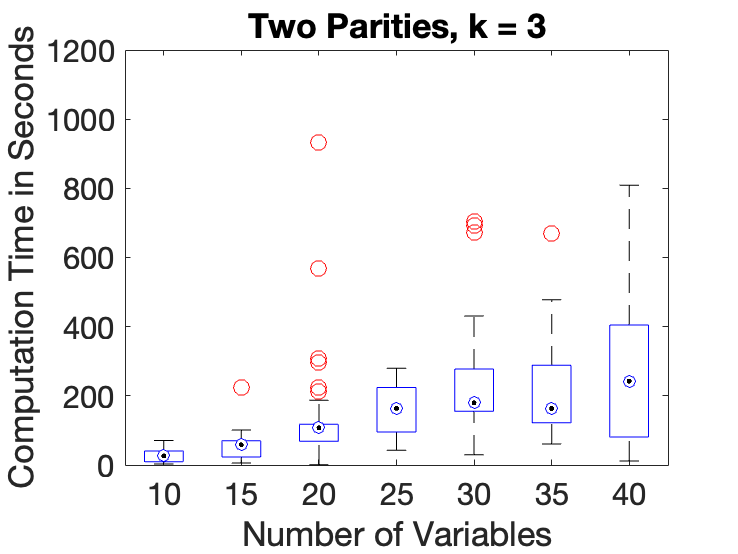} 
    \end{minipage}\hfill
    \begin{minipage}{0.49\textwidth}
\includegraphics[scale = .36]{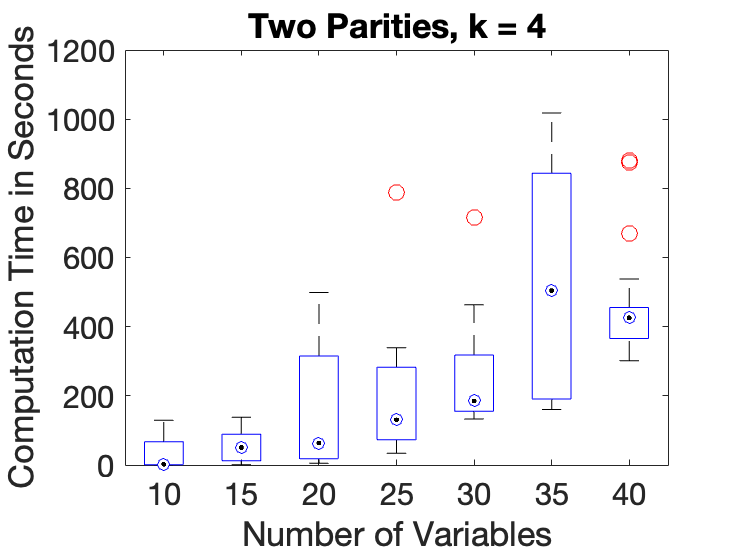} 
    \end{minipage}
    \caption{\label{figure:tp}We consider the performance of \Cref{alg:markov_blanket} on instances with two opposing parities. The algorithm succeeded in all instances, typically quite quickly.}
\end{figure}
\paragraph{Multiple Opposing Parities} We also provide preliminary evidence that \Cref{alg:markov_blanket} can easily succeed beyond SPN instances. We now plant two random parities containing $1$ with opposite $\pm 1$ signs in $\psi$ and consider the time to stably achieve approximate recovery as before. We plot the results in Figure~\ref{figure:tp}; the algorithm succeeded in all trials. The runtimes are mostly monotonic in both $k$ and $n$, and we expect that the small non-monotonicities arise from random fluctuations in our statistics and variations in the effective window size using our consistent scaling as above. Determining optimal window sizes that can be used for practical applications is an important direction for future work.


\end{document}